\documentclass[twoside]{article}

\usepackage[accepted]{aistats2024}
\usepackage[round]{natbib}

\usepackage[T1]{fontenc}
\usepackage[dvipsnames]{xcolor}
\usepackage[title]{appendix}

\usepackage[colorlinks=true,allcolors=Blue,backref=page]{hyperref}

\renewcommand*{\backref}[1]{}
\renewcommand*{\backrefalt}[4]{{%
    \ifcase #1 Not cited.%
          \or Cited on page~#2.%
          \else Cited on pages #2.%
    \fi%
    }}

\usepackage{url}
\usepackage{hyphenat}
\usepackage{booktabs}
\usepackage{graphicx,amsfonts,amsmath,mathtools,amssymb}

\usepackage{algcompatible}
\usepackage{algorithm}
\usepackage{float}
\usepackage{amsmath,amsfonts,bm}

\newcommand{\figleft}{{\em (Left)}}
\newcommand{\figcenter}{{\em (Center)}}
\newcommand{\figright}{{\em (Right)}}
\newcommand{\figtop}{{\em (Top)}}
\newcommand{\figbottom}{{\em (Bottom)}}

\def\eqref#1{equation~\ref{#1}}
\def\1{\bm{1}}

\def\vzero{{\bm{0}}}
\def\vone{{\bm{1}}}

\def\vphi{{\bm{\phi}}}

\def\va{{\bm{a}}}
\def\vb{{\bm{b}}}
\def\vc{{\bm{c}}}

\def\vf{{\bm{f}}}

\def\vq{{\bm{q}}}
\def\vr{{\bm{r}}}

\def\vu{{\bm{u}}}
\def\vv{{\bm{v}}}
\def\vw{{\bm{w}}}
\def\vx{{\bm{x}}}
\def\vy{{\bm{y}}}

\def\mA{{\bm{A}}}
\def\mB{{\bm{B}}}

\def\mI{{\bm{I}}}

\def\mK{{\bm{K}}}
\def\mL{{\bm{L}}}

\def\mU{{\bm{U}}}

\def\mW{{\bm{W}}}

\def\mLambda{{\bm{\Lambda}}}

\DeclareMathAlphabet{\mathsfit}{\encodingdefault}{\sfdefault}{m}{sl}
\SetMathAlphabet{\mathsfit}{bold}{\encodingdefault}{\sfdefault}{bx}{n}

\def\gG{{\mathcal{G}}}

\def\gN{{\mathcal{N}}}
\def\gO{{\mathcal{O}}}
\def\gP{{\mathcal{P}}}

\newcommand{\R}{\mathbb{R}}

\newcommand{\diag}{\mathrm{diag}}

\newcommand{\Var}{\mathrm{Var}}

\newcommand{\Cov}{\mathrm{Cov}}
\newcommand{\rmu}{\mathrm{u}}
\newcommand{\rmd}{\mathrm{d}}

\newcommand{\im}{\mathrm{im}}

\newcommand{\textGrad}{\mathrm{grad}}
\newcommand{\textCurl}{\mathrm{curl}}
\newcommand{\textDiv}{{\mathrm{div}}}

\newcommand{\textSC}{\mathrm{SC}_2}
\newcommand{\Matern}{Mat\'ern}
\mathchardef\mhyphen="2D % Define a "math hyphen"

\newcommand{\spn}{\mathrm{span}}

\usepackage[most]{tcolorbox}
\usepackage{subcaption}
\usepackage{caption}
\usepackage{wrapfig}
\usepackage{nicefrac}       % compact symbols for 1/2, etc.
\usepackage{tabstackengine}
\usepackage{derivative}
\usepackage[export]{adjustbox}
\usepackage{changepage}
\usepackage{multirow}

\setstackEOL{\cr}

\usepackage{amsthm}
\def\vu{{\bm{u}}}
\usepackage[capitalize,nameinlink]{cleveref}
\crefname{assumption}{Assumption}{Assumptions}
\theoremstyle{plain}
\newtheorem{theorem}{Theorem}
\newtheorem{proposition}[theorem]{Proposition}
\newtheorem{lemma}[theorem]{Lemma}
\newtheorem{corollary}[theorem]{Corollary}

\theoremstyle{definition}
\newtheorem{definition}[theorem]{Definition}
\theoremstyle{remark}
\newtheorem{remark}[theorem]{Remark}

\usepackage{kbordermatrix}
\renewcommand{\kbldelim}{(}% Left delimiter
\renewcommand{\kbrdelim}{)}% Right delimiter

\begin{document}

% \runningtitle{I use this title instead because the last one was very long}


\twocolumn[

\aistatstitle{Hodge-Compositional Edge Gaussian Processes}

\aistatsauthor{Maosheng Yang \And Viacheslav Borovitskiy \And  Elvin Isufi }

\aistatsaddress{TU Delft, Netherlands \And  ETH Zürich, Switzerland \And TU Delft, Netherlands } ]

\begin{abstract}
  We propose principled Gaussian processes (GPs) for modeling functions defined 
  over the edge set of a simplicial 2-complex, a structure similar to a graph in which edges may form triangular faces.
  This approach is intended for learning flow-type data on networks where edge flows can be characterized by the discrete divergence and curl.
  Drawing upon the Hodge decomposition, we first develop classes of divergence-free and curl-free edge GPs, suitable for various applications. 
  We then combine them to create \emph{Hodge-compositional edge GPs} that are expressive enough to represent any edge function. 
  These GPs facilitate direct and independent learning for the different Hodge components of edge functions, enabling us to capture their relevance during hyperparameter optimization.
  To highlight their practical potential, we apply them for flow data inference in currency exchange, ocean currents and water supply networks, comparing them to alternative models.
  % We investigate principled approaches to define Gaussian processes (GPs) on the edge set of a simplicial 2-complex---a structure like graph in which edges may form triangular faces---for flow-type data on networks.
  % Edge flows can be characterized by the discrete concepts of divergence and curl, measuring how they diverge at nodes and circulate along triangular faces.
  % In this paper, we first develop classes of divergence-free and curl-free edge GPs. 
  % We then combine them---taking inspiration from the Hodge decomposition of edge functions---to create a more expressive family of \emph{Hodge-compositional edge GPs} that can represent any edge functions.
  % This yields a principled prior that enables automatic relevance determination of different parts of the Hodge decomposition.
  % We validate its effectiveness in inferring flow data in currency exchange, ocean flows and water supply networks, comparing it to various alternatives.
\end{abstract}

\section{INTRODUCTION} \label{sec:introduction}
Gaussian processes (GPs) are a widely used class of statistical models capable of quantifying uncertainty associated to their own predictions \citep{rasmussenGaussianProcessesMachine2006}. 
These models are determined by covariance kernels which encode prior knowledge about the unknown function. 
Choosing an appropriate kernel is often challenging, particularly when the input space is non-Euclidean \citep{duvenaud2014automatic}.

Developing GPs on graphs has been a subject of recent work, which requires structured kernels to encode the dependence between nodes \citep{venkitaramanGaussianProcessesGraphs2018,zhiGaussianProcessesGraphs2023}, like the diffusion \citep{smolaKernelsRegularizationGraphs2003} or random walk kernels \citep{vishwanathan2010graph}. 
More recently, \citet{borovitskiyMatErnGaussian2021} derived the more general family of \Matern~kernels on graphs from stochastic partial differential equations (SPDEs) thereon, mirroring the continuous approaches on manifolds \citep{borovitskiy2020matern,azangulov2022stationary,azangulov2023stationary}.
\citet{nikitinNonseparableSpatiotemporalGraph2022} incorporated the temporal factor in this framework to build temporal-graph kernels.
However, GPs in these works are targeted for modeling functions on the nodes of networks. 

We instead focus on functions defined on the \emph{edges}, of particular interest for modeling edge-based dynamical processes in many complex networks, such as flows of energy, signal or mass \citep{schaub2014structure}.
For example, in water supply networks, we typically monitor the flow rates within pipes (edges) connecting tanks (nodes) \citep{zhou2022bridging}.
Other examples include energy flows in power grids \citep{jiaGraphbasedSemiSupervisedActive2019}, synaptic signals between neurons in brain networks \citep{faskowitz2022edges}, and exchange rates on trading paths (edges) of currencies (nodes) \citep{jiangStatisticalRankingCombinatorial2011}.

\begin{table}[b!]
\vspace*{-3.75ex}
\footnotesize
\urlstyle{same}
\rule{0.8in}{0.4pt}\\[0.75ex]
Correspondence to: \href{mailto:m.yang-2@tudelft.nl}{m.yang-2@tudelft.nl}. Code available at: \url{https://github.com/cookbook-ms/Hodge-Edge-GP}.$\!$
\end{table}

While it might seem intuitive to use node-based methods for edge-based tasks using line-graphs \citep{godsilAlgebraicGraphTheory2001}, this often yields sub-optimal solutions \citep{jiaGraphbasedSemiSupervisedActive2019}. 
Alternatively, recent successes in signal processing and neural networks for edge data have emerged from modeling flows on the edge set of a simplicial 2-complex ($\textSC$), including \citep{jiaGraphbasedSemiSupervisedActive2019,barbarossaTopologicalSignalProcessing2020,schaubSignalProcessingHigherOrder2021,yangSimplicialConvolutionalNeural2022,yangSimplicialConvolutionalFilters2022,roddenberry2021principled,yangConvolutionalLearningSimplicial2023}, among others. 
A $\textSC$ can be viewed as a graph with the additional set of triangular faces, encoding how edges are adjacent to each other via nodes or faces.
A $\textSC$ also allows to characterize key properties of edge flows using discrete concepts of \emph{divergence} (div) and \emph{curl} \citep{lovaszDiscreteAnalyticFunctions2004,limHodgeLaplaciansGraphs2020}, measuring how they diverge at nodes and circulate along faces.
For example, electric currents in circuit networks respecting the Kirchhoff's law are div-free \citep{gradyDiscreteCalculus2010}, and arbitrage-free exchange rates are curl-free along loops of trading paths \citep{jiangStatisticalRankingCombinatorial2011}.
Moreover, edge functions on a $\textSC$ admit the \emph{Hodge decomposition} into three parts: gradient, curl and harmonic components, being either curl-free, div-free or both div- and curl-free \citep{limHodgeLaplaciansGraphs2020}.
This provides unique insights in various applications including ranking \citep{jiangStatisticalRankingCombinatorial2011}, gaming theory \citep{candoganFlowsDecompositionsGames2011}, brain networks \citep{vijay2022hodge} and finance \citep{fujiwaraHodgeDecompositionBitcoin2020}.  
Nevertheless, existing works on edge-based learning remain deterministic and there is a lack of principled ways to define GP priors on the edge set of SCs, which is the central goal of this work.

% We start with deriving edge GPs from SPDEs on the edge space of SCs as a direct extension of graph counterparts.
Our main contribution lies in the proposal of \emph{Hodge-compositional edge GPs}.
We build them as combinations of three GPs, each modeling a specific part of the Hodge decomposition of an edge function, namely the gradient, curl and harmonic parts.
With a focus on the \Matern~GP family, we show that each of them can be linked to a SPDE, extending the framework used by \citet{borovitskiy2020matern,borovitskiyMatErnGaussian2021,borovitskiyIsotropicGaussianProcesses2023}.
Compared to a direct extension of graph GPs, they enable separate learning of the different Hodge components, which allows us to capture the practical behavior of edge flows.
We also demonstrate their practical potential in edge-based learning tasks in foreign currency exchange markets, ocean current analysis and water supply networks. 
% They can be built based on the Hodge Laplacian, encoding how edges are connected via nodes and triangular faces, which is a linear operator on edge functions, analogous to the graph Laplacian operating on function over nodes. 

% In our discussion, we occasionally draw some conceptual resemblance of the proposed edge GPs with GPs for vector fields \citep{baldassarre2010vector,berlinghieri2023gaussian}, since node and edge functions can be informally viewed as discrete analogs of scalar and vector fields in continuous domains. 
% However, we consider the edges and the SC are \emph{intrinsically} discrete---there is no underlying continuous field or domain that we seek to approximate.

% The rest of the paper is organized as follows: in \cref{sec:background} we first review GPs and GPs on graphs, as well as preliminaries on edge functions and Hodge Laplacians. 
% In \cref{sec:edge-gp}, we start with building edge GPs from SPDEs on edges, we then discuss the divergence- and curl-free GPs after introducing some notions on discrete calculus. 
% Later, we propose the Hodge-compositional GPs and discuss their properties.
% We conduct three experiments in \cref{sec:experiments}, followed by the conclusion and discussion. 

\section{BACKGROUND} \label{sec:background}
A random function $f:X\to\R$ defined over a set $X$ is a Gaussian process $f\sim \gG\gP(\mu,k)$ with mean function $\mu(\cdot)$ and kernel $k(\cdot,\cdot)$ if, for any finite set of points $\vx = (x_1,\dots,x_n)^\top\in X^n$, the random vector $f(\vx) = (f(x_1),\dots,f(x_n))^\top$ is multivariate Gaussian with mean vector $\mu(\vx)$ and covariance matrix $k(\vx,\vx)$.

The kernel $k$ of a \emph{prior} GP encodes prior knowledge about the unknown function while its mean $\mu$ is usually assumed to be zero.
GP regression combines such a \emph{prior} with training data $x_1,y_1,\dots,x_n,y_n$ where $x_i\in X$, $y_i\in\R$ with $y_i=f(x_i)+\epsilon_i$,  $\epsilon_i\sim\gN(0,\sigma_\epsilon^2)$. 
This results in a posterior $f_{\mid \vy}$ which is another GP: $f_{\mid \vy} \sim \mathcal{GP}(\mu_{\mid \vy}, k_{\mid \vy})$.
For any new input $x^* \in X$, the mean $\mu_{\mid \vy}(x^*)$ is the prediction and the posterior variance $ k_{\mid \vy}(x^*, x^*)$ quantifies the uncertainty.
% See \cref{app:gp-basics} for a recipe of GP regression.
We refer the reader to \citet{rasmussenGaussianProcessesMachine2006} for more details.
Defining an appropriate kernel is one of the main challenges in GP modeling \citep{duvenaud2014automatic}.
% On continuous domains like Euclidean spaces or manifolds, GPs with widely used \Matern~kernels, including squared exponential kernels, have been derived from their SPDE representation \citep{whittle1963stochastic,lindgren2011explicit,sarkka2011linear,borovitskiy2020matern}.

\subsection{GPs on Graphs} \label{subsec:graph-gp}
Let $G=(V, E)$ be an unweighted graph where $V=\{1,\dots,N_0\}$ is the set of nodes and $E$ is the set of $N_1$ edges such that if nodes $i,j$ are connected, then $e=(i,j)\in E$. 
We can define real-valued functions on its node set $f_0:V\to\R$, collected into a vector $\vf_0=(f_0(1),\dots,f_0({N_0}))^\top \in \R^{N_0}$.  
Denote the \emph{oriented} node-to-edge incidence matrix by $\mB_1$ of dimension $N_0\times N_1$.
It has entries $[\mB_1]_{ie}=-1$ and $[\mB_1]_{je}=1$, if an edge $e=\{i,j\}$ exists, and zero otherwise. 
One can view oriented graphs as undirected graphs having an additional orientation structure on the edge set \citep[Sec 8.3]{godsilAlgebraicGraphTheory2001}.
They are different from directed graphs, as discussed next subsection.
The \emph{graph Laplacian} is then given by $\mL_0=\mB_1\mB_1^\top$, which is a positive semi-definite linear operator on the space $\R^{N_0}$ of node functions.
It admits an eigendecomposition $\mL_0 = \mU_0\mLambda_0\mU_0^\top$ where $\mLambda_0$ collects its eigenvalues on the diagonal and $\mU_0$ collects the orthogonal eigenvectors of $\mL_0$ \citep{chung1997spectral}.

A GP on graphs $\vf_0\sim \gG\gP(\vzero, \mK_0)$ assumes $\vf_0$ is a random function with zero mean and a graph kernel $\mK_0$ which encodes the covariance between pairs of nodes. 
To construct principled graph GPs, \citet{borovitskiyMatErnGaussian2021} extended the idea of deriving continuous GPs from SPDEs \citep{whittle1963stochastic,lindgren2011explicit} to the domain of graphs. 
Specifically, given the following SPDE on graphs with a Gaussian noise $\vw_0\sim\gN(\vzero,\mI)$
\begin{equation} \label{eq.graph-spde}
  \Phi({\mL}_0) \vf_0 = \vw_0, \text{ with } \Phi({\mL}_0) =   \Bigl( \frac{2\nu}{\kappa^2}\mI+\mL_0 \Bigr)^\frac{\nu}{2},
\end{equation}
where $\Phi(\mL_0) = \mU_0\Phi(\mLambda_0)\mU_0^\top$ and $\Phi(\cdot)$ applies to $\mLambda_0$ element-wise, its solution is the \Matern~graph GP
\begin{equation} \label{eq.graph-matern-gp}
  \vf_0 \sim \gG\gP   \Big( \vzero, \Bigl( \frac{2\nu}{\kappa^2}\mI+\mL_0 \Bigr)^{-\nu} \Big)
\end{equation}
with positive parameters $\kappa,\nu$. 
When scaled properly, the \Matern~kernel gives the graph diffusion kernel for $\nu\to\infty$, which in turn relates to the random walk kernel by \citet{kondor2002diffusion}. 
% Furthermore, for $\nu=2,\kappa\to\infty$, the \Matern~kernel reduces to the graph Laplacian kernel, where the SPDE in \cref{eq.graph-spde} returns to the graph Laplace's SPDE: $\mL_0\vf_0=\vw_0$. 
% In the case of graph as a discretization of a manifold, \citet{sanz2022spde} showed the \Matern~GP converges to its continuous counterpart in Riemannian manifolds by \citet{borovitskiy2020matern}. 
This SPDE framework can be extended to spatial-temporal data yielding respective graph kernels \citep{nikitinNonseparableSpatiotemporalGraph2022}.

\begin{figure*}[hpt!]
  %\vspace{-3mm}
  \makebox[\textwidth][c]{
  \begin{subfigure}{0.2\linewidth}
    \includegraphics[width=1.1\linewidth]{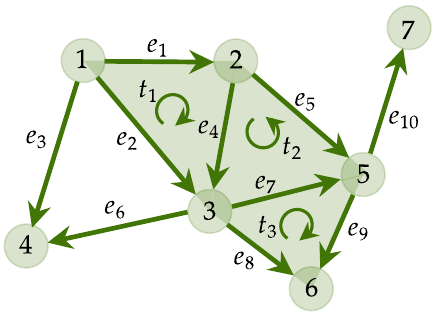}
    \caption{$\textSC$}
    \label{fig:sc_example}
  \end{subfigure}
  \begin{subfigure}{0.2\linewidth}
    \includegraphics[width=1.1\linewidth]{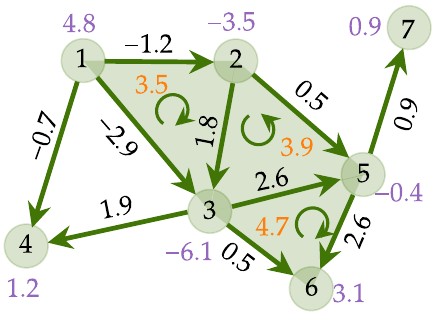}
    \caption{$\vf_1$}
    \label{fig:sc_edge_function}
  \end{subfigure}
  \hspace{7.5pt}
  \begin{subfigure}{0.2\linewidth}
    \includegraphics[width=1.1\linewidth,right]{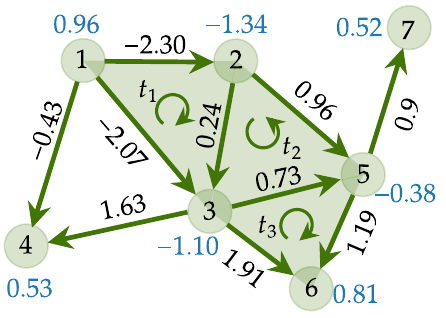}
    \caption{$\vf_G$}
    \label{fig:sc_edge_gradient}
  \end{subfigure}
  \begin{subfigure}{0.2\linewidth}
    \includegraphics[width=1.1\linewidth,right]{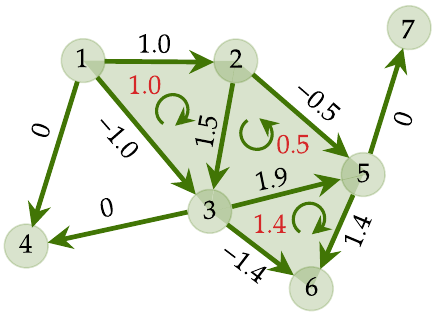}
    \caption{$\vf_C$}
    \label{fig:sc_edge_curl}
  \end{subfigure}
  \begin{subfigure}{0.2\linewidth}
    \includegraphics[width=1.1\linewidth,right]{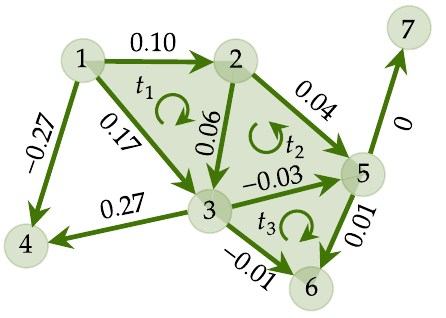}
    % \vspace{-5mm}
    \caption{$\vf_H$}
    \label{fig:sc_edge_harmonic}
  \end{subfigure}
  }
  %\vspace{-5mm}
  \caption{(a) A $\textSC$ where we shade (closed) triangles in green and denote reference orientations of edges/triangles by arrows. (b) An edge function $\vf_1$ with its divergence (\textcolor{Purple}{purple} values on nodes) and curl (\textcolor{orange}{orange} values in triangles). (c-e) Hodge decomposition: (c) gradient flow $\vf_G=\mB_1^\top\vf_0$, obtained as the gradient of some node function $\vf_0$ (given in \textcolor{RoyalBlue}{blue}). It is curl-free: the net-circulation along each triangle is zero; (d) curl flow $\vf_C=\mB_2\vf_2$, induced by some circulating triangle signal $\vf_2$ (given in \textcolor{red}{red}). It is div-free: the net-flow at each node is zero; and (e) harmonic flow $\vf_H$, circulating around the 1-dimensional ``hole'' (open triangle $\{1,3,4\}$), where the net-flow on nodes and net-circulation in triangles are zero. All numbers are rounded to two decimal places.}
  \label{fig:hodge_decomposition_illustration}
  %\vspace{-3mm}
\end{figure*}

\subsection{Edge Functions on Simplicial Complexes}
Simplicial 2-complexes  represent discrete geometry more expressively than graphs.
They are triples $\textSC = (V,E,T)$ where $V, E$ are the sets of nodes and edges, same as for graphs, and $T$ is the set of triangular faces (shortened as triangles) such that if $(i,j),(j,k),(i,k)$ form a \emph{closed} triangle, then $t=(i,j,k)\in T$ \citep{munkresElementsAlgebraicTopology2018}.
Note that \emph{not} all three pairwise connected edges are necessarily closed and included in $T$.
An example of a $\textSC$ is shown in \cref{fig:sc_example} where the set $\{1,3,4\}$ is an {open} triangle and thus is not in $T$.

For each edge and each triangle, we assume the increasing order of their node labels as their reference \emph{orientation}\footnote{The orientation of a general simplex is an equivalence class of permutations of its labels---two orientations are equivalent (respectively, opposite) if they differ by an even (respectively, odd) permutation \cite[Sec 4]{limHodgeLaplaciansGraphs2020}.}.
An oriented edge, denoted as $e=[i,j]$, is an ordering of $\{i,j\}$.
This is not a directed edge allowing flow only from $i$ to $j$, but rather an assignment of the sign of the flow: from $i$ to $j$ it is positive and the reverse is negative.
Same goes for oriented triangles denoted as $t=[i,j,k]$ and we have $[i,j,k]=[j,k,i]=[k,i,j]=\!\!-[i,k,j]=\!\!-[k,j,i]=\!\!-[j,i,k]$.

% \vspace{-2mm}
% \paragraph{Functions on Edges}
% Given a $\textSC$, we may also define real-valued functions on edges $E$ but we shall require them to be alternating\footnote{We refer to \citet{limHodgeLaplaciansGraphs2020} for a formal definition.}. 
In a $\textSC$, the functions, $f_1:E\to \R$, on its edges $E$ are required to be \emph{alternating} \citep{limHodgeLaplaciansGraphs2020}, meaning that, we have $f_1(\bar{e}) = -f_1(e)$ if $\bar{e}=[j,i]$ is oriented opposite to the reference $e=[i,j]$. 
For example, in \cref{fig:sc_edge_function}, $f_1(1,2)=-1.2$ means there is a $1.2$ unit of flow from 2 to 1.
% Thus, if $f_1(e)$ is negative, then the flow is opposite to the orientation of $e$.
This property keeps the flow unchanged with respect to the edge orientation.
We collect the edge functions on $E$ into $\vf_1 = (f_1(e_1),\dots,f_1(e_{N_1}))^\top\in\R^{N_1}$, as in \cref{fig:sc_edge_function}, which we also call as an \emph{edge flow}. 

We can also define alternating functions on triangles in $T$ where $f_2(\bar{t}) = -f_2(t)$ if $\bar{t}$ is an odd permutation of reference $t=[i,j,k]$ \citep{limHodgeLaplaciansGraphs2020}. 
We collect them in $\vf_2 \in \R^{N_2}$ where $N_2=|T|$.
% One can endow the Hilbert spaces $L^2(V), L^2(E), L^2(T)$ with a $\rm{SC}_2$. 
In topology, functions $f_0, f_1, f_2$ are called \emph{0-, 1-, 2-cochains}, which are discrete analogs of differential forms on manifolds \citep{gradyDiscreteCalculus2010}. 
This motivates the use of subscripts $0, 1, 2$.
Here we can view these functions as vectors of data on nodes, edges and triangles.

\subsection{Hodge Laplacians} 
In the similar spirit as $\mL_0$ operating on node functions, we can define the discrete \emph{Hodge Laplacian} operating on the space $\R^{N_1}$ of edge functions
\begin{equation}
  \mL_1 = \mB_1^\top\mB_1 + \mB_2\mB_2^\top := \mL_{\rmd} + \mL_{\rmu}
\end{equation} 
where $\mB_2$ is the edge-to-triangle incidence matrix.
Its entries are $[\mB_2]_{et} = 1$, for $e=[i,j] \text{ or } e=[j,k]$, and $[\mB_2]_{et} = -1$ for $e=[i,k]$, if a triangle $t=[i,j,k]$ exists, and zero otherwise. 
Matrix $\mL_1$ describes the connectivity of edges where the \emph{down} part $\mL_{\rmd}$ and the \emph{up} part $\mL_{\rmu}$ encode how edges are adjacent, respectively, through nodes and via triangles. 
For example, $e_3$ and $e_6$ are down neighbors sharing node $4$ in \cref{fig:sc_example} and $e_1$ and $e_2$ are up neighbors, collocated in $t_1$. 
Matrix $\mL_1$ is positive semi-definite, admitting an eigendecomposition $\mL_1=\mU_1\mLambda_1\mU_1^\top$ where diagonal matrix $\mLambda_1 = \diag(\lambda_1,\dots,\lambda_{N_1})$ collects the eigenvalues and $\mU_1$ is the eigenvector matrix.
Likewise, one can define $\mL_2 = \mB_2^\top\mB_2$ encoding the adjacency between triangles.  
Our discussion henceforth considers the unweighted $\mL_1$ but it also holds for the weighted variants in \citet{gradyDiscreteCalculus2010,schaubRandomWalksSimplicial2020}.

\section{EDGE GAUSSIAN PROCESSES}\label{sec:edge-gp}
We now define GPs on edges of a $\textSC$, specifically, $\vf_1 \sim \gG\gP(\vzero, \mK_1)$ with zero mean and edge kernel $\mK_1$. 
Throughout this work, we refer to them as \emph{edge GPs}, and call graph GPs in \cref{subsec:graph-gp} as \emph{node GPs} because they are both multivariate Gaussian but the former is indexed by $X=E$ and the latter by $X=V$. 
We start with deriving edge GPs from SPDEs on edges as a natural extension of \cref{eq.graph-spde}. 
Then, by introducing basic notions from discrete calculus \citep{gradyDiscreteCalculus2010} and the Hodge decomposition theorem, we propose the divergence-free and curl-free GPs, combining them into Hodge-compositional GPs.

\subsection{Edge GPs from SPDEs on Edges} \label{subsec:non-hodge-gp}
The derivation of graph GPs in \cref{eq.graph-matern-gp} as solutions of graph SPDE in \cref{eq.graph-spde} motivates the following SPDEs on edges, with edge Gaussian noise $\vw_1 \sim \gN(\vzero, \mI)$,  
% \begin{equation} \label{eq.spde-edge}
%   \Phi(\mL_1)\vf_1 = \vw_1, \text{ with } \Phi(\mL_1) = \mU_1 \Phi(\mLambda_1)\mU_1^\top,
% \end{equation}
% where function $\Phi(\cdot)$ applies to diagonal $\mLambda_1$ element-wise.
\begin{equation} \label{eq.spde-edge}
  \Phi(\mL_1)\vf_1 = \vw_1  
\end{equation}
where $\Phi(\mL_1) = \mU_1 \Phi(\mLambda_1)\mU_1^\top$ is a differential operator defined through $\mL_1$.
When we consider the operators 
\begin{equation} \label{eq.diff-opreators-matern-diffusion}
  \Phi(\mL_1) =  \Bigl( \frac{2\nu}{\kappa^2}\mI+\mL_1 \Bigr)^\frac{\nu}{2} \text{ and } \Phi(\mL_1) = e^{\frac{\kappa^2}{4}\mL_1},
\end{equation}  
the solutions to \cref{eq.spde-edge} give two edge GPs 
% \begin{subequations}\label{eq.simple_matern_diffusion}
%   \begin{align}
%     \vf_1 & \sim \gG\gP \Bigl(\vzero, \Bigl( \frac{2\nu}{\kappa^2}\mI+\mL_1 \Bigr)^{-\nu} \Bigr), \\
%     \vf_1 & \sim \gG\gP \Bigl(\vzero,e^{-\kappa\mL_1}\Bigr)
%   \end{align}
% \end{subequations}
% \makebox[0.5\textwidth][c]{
\begin{equation}\label{eq.simple_matern_diffusion}
\begin{aligned}
    \vf_{1,\text{\Matern}} &\sim \gG\gP \Bigl(\vzero, \Bigl( \frac{2\nu}{\kappa^2}\mI+\mL_1 \Bigr)^{-\nu} \Bigr), \\
    \vf_{1,\rm{diffusion}} &\sim \gG\gP \Bigl(\vzero,e^{-\frac{\kappa^2}{2}\mL_1}\Bigr),
\end{aligned}
\end{equation}
% }
which are the \emph{edge \Matern} and \emph{diffusion} GPs, respectively.
% This framework inherits the analysis of the graph counterpart in \citet{borovitskiyMatErnGaussian2021}.
These edge GPs impose structured prior covariance that encodes the dependence between edges.
A related \emph{Hodge Laplacian kernel} $(\mL_1^\top\mL_1)^\dagger$ can be obtained by setting $\Phi(\mL_1) = \mL_1$, i.e., $\mL_1 \vf_1 = \vw_1$. 
This kernel was used to penalize the smoothness of edge functions in \citet{schaubSignalProcessingHigherOrder2021}.
The kernels of \cref{eq.simple_matern_diffusion} are more flexible though and allow encoding non-local edge-to-edge adjacency while $\mL_1$ instead encodes the local direct (one-hop) adjacency.

\subsection{Div-free and Curl-free Edge GPs} \label{subsec:divcurl_free_gps}
The edge GPs in~\Cref{subsec:non-hodge-gp} define distributions over all edge functions.
As opposed to this, here we seek to define GPs on the classes of divergence-free and curl-free edge functions.
We start with defining the appropriate notions of discrete derivatives, expressed in terms of the incidence matrices.
% They can be measured in terms of notions of discrete derivatives, which are discrete analogs of $\textGrad, \textCurl$ and $\textDiv$ operators in calculus and related to the incidence matrices.

\paragraph{Discrete Derivatives}
The \emph{gradient} is a linear operator from the space of node functions to that of edge functions. At edge $e=[i,j]$, it is defined as 
\begin{equation}
  (\textGrad \, f_0)(e) = (\mB_1^\top\vf_0)_e = f_0(j) - f_0(i), 
\end{equation}
which computes the difference between the values of a function on adjacent nodes, resulting in a flow on the connecting edge.
We call $\vf_G=\mB_1^\top\vf_0$ a gradient flow and $\vf_0$ a node potential, as shown in \cref{fig:sc_edge_gradient}.

The \emph{divergence}, the adjoint of gradient, is a linear operator from the space of edge functions to that of node functions. At node $i$, it is defined as 
\begin{equation}
  (\textDiv \, f_1)(i) = (\mB_1\vf_1)_i = -  \sum_{j\in N(i)} f_1(i,j) 
\end{equation}
with $N(i)$ the neighbors of $i$. Physically, it computes the net-flow of edge functions passing through node $i$, i.e., the in-flow minus the out-flow, as shown in \cref{fig:sc_edge_function}. 
A \emph{divergence-free} flow has a zero net-flow everywhere. 

Lastly, the \emph{curl} operator is a linear operator from the space of edge functions to that of triangle functions. At triangle $t=[i,j,k]$, it is defined as 
\begin{equation}
% \begin{aligned}
    \hspace{-7pt} (\textCurl \, f_1)(t) = (\mB_2^\top\vf_1)_t  = f_1(i,j) + f_1(j,k) - f_1(i,k)
% \end{aligned}
\end{equation}
which computes the \emph{net-circulation} of edge functions along the edges of $t$, as a rotational measure of $\vf_1$, as shown in \cref{fig:sc_edge_function}. 
A \emph{curl-free} flow has zero curl over each triangle. 
As in calculus, we have the identity $\textCurl\,\textGrad = \mB_2^\top\mB_1^\top = \vzero$, i.e., gradient flow is curl-free.

% Div and curl allow us to characterize edge functions.
Analogous to their continuous vector field counterparts, div-free and curl-free edge functions are ubiquitous, e.g., the electric currents and the exchange rates later in \cref{subsec:forex_experiments}. 
We refer to \citet{gradyDiscreteCalculus2010,limHodgeLaplaciansGraphs2020} for more examples. 
From this perspective, we can view the graph Laplacian as $\mL_0 = \textDiv\, \textGrad = \mB_1\mB_1^\top$, which is a graph-theoretic analog of the Laplace-Beltrami operator $\Delta_0$ on manifolds. Also, the SPDE on graphs in \cref{eq.graph-spde} is a discrete counterpart of the continuous one for scalar functions on manifolds. 
Moreover, the Hodge Laplacian $\mL_1$ can be viewed as $\mL_1 = \textGrad\,\textDiv + \textCurl^*\,\textCurl =  \mB_1^\top\mB_1 + \mB_2\mB_2^\top$, which is a discrete analog of the vector Laplacian (or Helmholtzian) $\Delta_1$ for vector fields.
% Also, the SPDE on edges in \cref{eq.spde-edge} relates to the discrete version of the SPDE for vector fields, e.g., the Laplace's equation describing the vector field both irrotational and solenoidal \citep{evans2022partial}.

% \begin{remark}
%   Note that these notions are \emph{discrete}, not \emph{discretized}, in the sense that the domain $\textSC$ is intrinsically discrete, not necessarily a discretization of some underlying continuous domain. 
% \end{remark}
%\vspace{-2mm}
\paragraph{Hodge Decomposition}
The following \emph{Hodge decomposition theorem}, unfolding an edge function, will allow us to improve the edge GPs in \cref{eq.simple_matern_diffusion}. 
\begin{theorem}[\citet{hodge1989theory}] \label{thm:hodge-them}
  The space $\R^{N_1}$ of edge functions is a direct sum of three subspaces
  % \begin{equation} \label{eq.hodge-decomposition}
  %   \R^{N_1} = \lefteqn{\overbrace{\phantom{\im(\mB_1^\top) \oplus \ker(\mL_1)}}^{\ker(\mB_1): \rm{curl\mhyphen free}}}
  %   \im(\mB_1^\top) \oplus 
  %   \lefteqn{\underbrace{\phantom{\ker(\mL_1) \oplus \im(\mB_2)}}_{\ker(\mB_2^\top): \rm{div\mhyphen free}}} \ker(\mL_1)\oplus \im(\mB_2),
  % \end{equation}
  \begin{equation} \label{eq.hodge-decomposition}
    \R^{N_1} = 
    \im(\mB_1^\top) \oplus \ker(\mL_1)\oplus \im(\mB_2),
  \end{equation}
  % where $\im(\mB_1^\top)$ is the gradient space, $\ker(\mL_1)$ the harmonic space\footnote{Its dimension is the 1-dim Betti number in topology.} and $\im(\mB_2)$ the curl space.   
  where $\im(\mB_1^\top)$ is the gradient space, $\ker(\mL_1)$ the harmonic space and $\im(\mB_2)$ the curl space.   
\end{theorem}
It states that any edge function $\vf_1$ is composed of three orthogonal parts: gradient, curl, harmonic functions  
  \begin{equation} \label{eq.hodge-composition-expression}
    \vf_1 = \vf_G + \vf_H + \vf_C
  \end{equation}
where $\vf_G=\mB_1^\top\vf_0$, being curl-free, is the gradient of some node function $\vf_0$, and $\vf_C=\mB_2\vf_2$, being div-free, is the curl-adjoint of some triangle function $\vf_2$. Lastly, $\vf_H$ is harmonic (both div- and curl-free, $\mL_1\vf_H=\vzero$).
This decomposition is illustrated in \cref{fig:hodge_decomposition_illustration}.
It provides a crucial tool for understanding edge functions and has been used in many applications as we discussed above. 

Furthermore, the eigenspace $\mU_1$ of $\mL_1$ can be reorganized in terms of the three Hodge subspaces as  
\begin{equation} \label{eq.U_Hodge}
  \begin{aligned}
    \mU_1 = [\mU_H\,\,\mU_G\,\,\mU_C]
  \end{aligned}
\end{equation}
where $\mU_H$ is the eigenvector matrix associated to zero eigenvalues $\mLambda_H=\vzero$ of $\mL_1$, $\mU_G$ is associated to the nonzero eigenvalues $\mLambda_G$ of $\mL_{\rmd}$, and $\mU_{C}$ is associated to the nonzero eigenvalues $\mLambda_C$ of $\mL_{\rmu}$.
Moreover, they span the Hodge subspaces:
\begin{equation}
\begin{aligned}
\spn(\mU_H) &= \ker(\mL_1), \,\, \spn(\mU_G) = \im(\mB_1^\top), \\
\spn(\mU_C) &= \im(\mB_2),
\end{aligned}
\end{equation}
where $\spn(\bullet)$ denotes all possible linear combinations of columns of $\bullet$ \citep{yangSimplicialConvolutionalFilters2022}.
% This was used by \citet{yangFiniteImpulseResponse2021,yangSimplicialConvolutionalFilters2022} for spectral edge filtering. 

\paragraph{Div-free, Curl-free Edge GPs} Given the eigendecomposition in \cref{eq.U_Hodge}, we can obtain special classes of edge GPs by only using a certain type of eigenvectors when building edge kernels of \cref{eq.simple_matern_diffusion}. 
Specifically, we define \emph{gradient} and \emph{curl edge GPs} as follows 
\begin{equation} \label{eq.grad-curl-gps}
  \vf_G \sim \gG\gP(\vzero, \mK_G),  \quad \vf_C \sim \gG\gP(\vzero, \mK_C) 
\end{equation}
where the gradient kernel and the curl kernel are 
\begin{equation}
  \mK_G = \mU_G \Psi_G(\mLambda_G) \mU_G^\top, \,\, \mK_C = \mU_C \Psi_C(\mLambda_C) \mU_C^\top. 
\end{equation}
We also define the \emph{harmonic GPs} $\vf_H\sim\gG\gP(\vzero,\mK_H)$ with the harmonic kernel $\mK_H=\mU_H\Psi_H(\mLambda_H)\mU_H^\top$.
% \begin{proposition} \label{prop:samples-of-grad-curl-gps}
%     Let $\vf_G$ and $\vf_C $ be the gradient and curl Gaussian processes in \cref{eq.grad-curl-gps}, respectively. 
%     Then, their samples are curl-free and divergence-free, respectively. Moreover, samples of a harmonic Gaussian process are both div- and curl-free.
% \end{proposition}
\begin{proposition} \label{prop:samples-of-grad-curl-gps}
    Let $\vf_G$ and $\vf_C $ be the gradient and curl Gaussian processes, respectively. 
    Then, $\textCurl\,\vf_G = \vzero$ and $\textDiv\,\vf_C=\vzero$ with probability one. 
    Moreover, a harmonic Gaussian process $\vf_H$ follows $\textCurl\,\vf_H = \vzero$ and $\textDiv\,\vf_H=\vzero$ with probability one. 
\end{proposition}
See proof in \cref{proof:prop-gp-samples}.
These Hodge GPs provide more targeted priors for special edge functions which are either div- or curl-free, capable of capturing these key properties.
In the case of \Matern~kernels, we set 
\begin{equation}  \label{eq.grad-curl-matern-kernel}
  \Psi_\Box(\mLambda_\Box) = \sigma_\Box^2 \Big(\frac{2\nu_\Box}{\kappa_\Box^2}\mI + \mLambda_\Box \Big)^{-\nu_\Box}, 
\end{equation}
for $\Box\in\{H,G,C\}$, where $\sigma_\Box^2$ controls the variance we assign to the function in the subspace, and $\nu_\Box, \kappa_\Box$ are the regular \Matern~parameters. 
Note that since $\mLambda_H=\vzero$, we consider a scaling function for $\mK_H$ as $\Psi_H(\vzero) = \sigma_H^2$. 
We illustrate such a \Matern~kernel function in \cref{fig:spectral_kernel_illustration} (left). 
These Hodge GPs can be derived from SPDEs on edges as well.
\begin{proposition}\label{prop:spde-grad-curl-gp}
    Given a scaled curl white noise  $\vw_C\sim\gN(\vzero,\mW_C)$ where $\mW_C = \sigma_C^2\mU_C\mU_C^\top$, consider the following SPDE on edges: 
    \begin{equation} \label{eq.grad-curl-spde}
      \Phi_C(\mL_{\rmu}) \vf_C = \vw_C, 
    \end{equation}
    with differential operators
    \begin{equation}
        \Phi_C(\mL_{\rmu}) =   \Big(\frac{2\nu_C}{\kappa_C^2} \mI + \mL_{\rmu} \Big)^{\frac{\nu_C}{2}}, \,\, \Phi_C(\mL_{\rmu})=e^{\frac{\kappa_C^2}{4}\mL_{\rmu}}.
    \end{equation}
    The respective solutions give the curl edge GPs with \Matern~kernel in \cref{eq.grad-curl-matern-kernel} and diffusion kernel
    \begin{equation} \label{eq.diffusion_kernel_spectra}
      \Psi_C(\mLambda_C) = \sigma_C^2 e^{-\frac{\kappa_C^2}{2} \mLambda_C} .
    \end{equation}
    Likewise, we can derive the gradient \Matern~and diffusion GPs from the SPDEs as \cref{eq.grad-curl-spde} but with operators $\Phi_G(\mL_{\rmd})$ and a scaled gradient white noise. 
\end{proposition}
See proof in  \cref{app:derivation-grad-curl-gps}. 
We can draw the intuition of SPDE in \cref{eq.grad-curl-spde} from the continuous analogy. 
In the case of $\mL_{\rmu}\vf_C=\vw_C$, the equation $\textCurl^* \textCurl \, f_1(\vx) = w_1(\vx)$ is a stochastic vector Laplace's equation of a div-free (solenoidal) vector field, where $w_1(\vx)$ the curl adjoint of some vector potential. 
In physics, this describes the static magnetic field from a magnetic vector potential, as well as an incompressible fluid.  
\begin{figure}[t!]
  %\vspace{-1mm}
  \centering
  \includegraphics[width=1\linewidth]{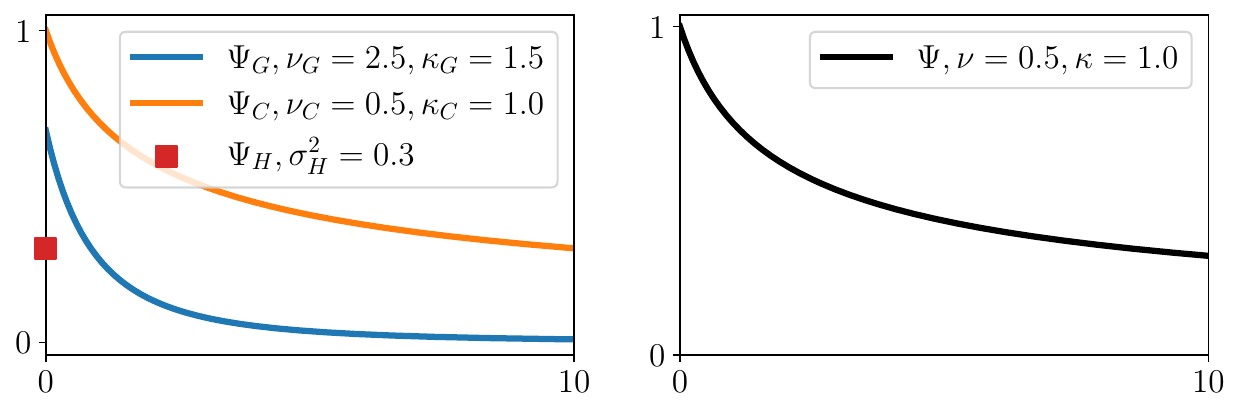}
  %\vspace{-6mm}
  \caption{\figleft~\Matern~kernel functions $\Psi_{\Box}(\lambda)$ for $\Box=\{H,G,C\}$ in \cref{eq.grad-curl-matern-kernel} of gradient, curl and harmonic GPs in the eigen-spectrum $\lambda$ ranging in the min and man eigenvalues of $\mL_1$. \figright~\Matern~kernel function $\Psi(\lambda)$ of non-HC GP in \cref{eq.simple_matern_diffusion}. }
  \label{fig:spectral_kernel_illustration}
  % \vspace{-3mm}
\end{figure}

\subsection{Hodge-compositional Edge GPs} \label{subsec:hc-gp}
% We showed how to construct special edge GPs to represent div-/curl-free functions.
Many edge functions of interest are indeed div- or curl -free, but not all.
% While they are of particular interest for edge functions that are in certain subspace, general edge functions do not always have such properties. 
In this section we combine the gradient, curl and harmonic GPs to define the Hodge-compositional (HC) edge GPs.
% Importantly, these have a separate set of hyperparameters for each of the aforementioned components, allowing them to recover the full edge Matérn GPs as well as their div-/curl-free versions with appropriate setting of hyperparameters.

% , given by the following definition. 
\begin{definition} \label{def:hodge-compositional-gp}
  A Hodge-compositional edge Gaussian process $\vf_1\sim \gG\gP(\vzero, \mK_1)$ is a sum of gradient, curl and harmonic GPs, i.e., 
  $\vf_1 = \vf_G + \vf_C + \vf_H$
  where    
  \begin{equation} \label{eq.hodge-gp-per-comp}
    \vf_{\Box } \sim \gG\gP(\vzero, \mK_{\Box}) \text{ with } \mK_{\Box} = \mU_{\Box}\Psi_{\Box}(\mLambda_{\Box})\mU^\top_{\Box}
  \end{equation}
  for $\Box = H,G,C$ where their kernels do not share hyperparameters.
\end{definition}
Given this definition, we can obtain the following properties of HC edge GPs.
\begin{lemma} \label{lemma:property-hc-gp}
    Let $\vf_1\sim\gG\gP(\vzero,\mK_1)$ be an edge GP in \cref{def:hodge-compositional-gp}. Its realizations then give all possible edge functions\footnote{Note that HC edge GPs do not represent all possible edge GPs. They are a particular GP family satisfying the independence hypothesis on the three Hodge GPs. This, however, does not contradict with that their realizations can represent all edge functions.}. It further holds that $\mK_1 = \mK_H + \mK_G + \mK_C$, and the three Hodge GPs are mutually independent, i.e., $\text{Cov}(\bm{f}_G, \bm{f}_C)=  \text{Cov}(\bm{f}_G,\bm{f}_H) = \text{Cov}(\bm{f}_C,\bm{f}_H) =\bm{0}$. 
\end{lemma}
See proof in \cref{app:proof-properties-hc-gp}. 
Naturally, we can construct a \Matern~HC GP as the sum of \Matern~GPs in the three subspaces with their kernels given by \cref{eq.grad-curl-matern-kernel}, and likewise for the diffusion HC GP by \cref{eq.diffusion_kernel_spectra}.
% Other computational techniques in \citet{borovitskiyMatErnGaussian2021} can be adopted as well.
Compared to the GPs in \cref{eq.simple_matern_diffusion}, referred to as non-HC GPs henceforth, HC GPs are more flexible and expressive, having more degrees of freedom.
We discuss their practical advantages below. 

\vspace{-1mm}
\paragraph{Inductive GP Prior}
The HC GP encodes the prior covariance $\Cov(f_1(e), f_1(e^\prime))$ between edge functions over two edges $e, e^\prime$ as follows: 
(i) the covariance is the sum of three covariances $\Cov_\Box=\Cov(f_\Box(e), f_\Box(e^\prime))$ for $\Box=H,G,C$; 
(ii) each $\Cov_\Box$ encodes the covariance between the corresponding Hodge parts of $f_1$ without affecting the others; and 
(iii) no covariance is imposed across different Hodge components, e.g., $\Cov(f_G(e),f_C(e^\prime))=0$.

In the spatial/edge domain, this is related to separating the down and up adjacencies encoded in the SPDE operators $\Phi(\cdot)$.
From an eigen-spectrum perspective, the eigenvalues $\Psi_\Box$ of HC GP's kernels associated to the three Hodge subspaces have individual parameters.
This enables capturing the different Hodge components of edge functions, as well as their relevance during hyperparameter optimization, further allowing us to directly recover the Hodge components in predictions, which we detail in \cref{app-subsec:posterior-hodge-components}.
Non-HC GPs instead require solving the Hodge decomposition in \cref{eq.hodge-composition-expression} (least squares problems) \citep{limHodgeLaplaciansGraphs2020}.
Another implication is that, unlike for non-HC GPs, we do not require specific knowledge about the div or curl of the underlying function. 
% Also, by considering $\mK_G$ in the posterior mean instead of $\mK_1$, we can obtain the gradient part of predictions for analysis purpose, and likewise for the other parts. 
% Also, we can obtain the Hodge components of the prediction from the optimized corresponding kernel in regression problems. 

\vspace{-1mm}
\paragraph{Comparison to non-HC GPs}
When we view non-HC GPs in terms of the Hodge decomposition, we notice that they put priors on the three Hodge GPs in a way that shares hyperparameters.
This enforces learning the same hyperparameters for different Hodge components, resulting in a single function covering the entire edge spectrum, as shown in \cref{fig:spectral_kernel_illustration} (right), as opposed to the three individual functions of the HC one.

% Assume the underlying function is a gradient flow. 
% A good edge GP should have both $\Psi_C(\lambda_C)=0$ and $\Psi_G(\lambda_G)\neq 0$ so to generate samples that are curl-free but not div-free, which is however difficult for non-Hodge GPs. 
This raises issues when separate learning, say, different lengthscales, is required for the gradient and curl components. 
Non-HC GPs are strictly incapable of this practical need when an eigenvalue is associated to both gradient and curl spaces.  
We also delve into this in terms of \emph{edge Fourier features} in \cref{app:fourier-feature-perspective} where we compare the priors induced on the edge Fourier coefficients by HC and non-HC GPs.

%\vspace{-2mm}
\paragraph{Connection to Diffusion on Edges}
The HC diffusion kernel, given by $\mK_1 = \exp({-(\frac{\kappa_G^2}{2} \mL_{\rmd} + \frac{\kappa_C^2}{2} \mL_{\rmu})})$, when $\sigma^2_\Box$s are one, is the Green's function for the edge diffusion of a function $\vphi:[0,\infty)\times E\to \R$
\begin{equation}
      \odv{\vphi(t)}{t} =  -(\mu \mL_{\rmd} + \gamma \mL_{\rmu}) \vphi(t), \text{ where }\mu,\gamma>0
\end{equation}
with $\vphi|_{t=\tau} = e^{-(\mu\tau\mL_\rmd + \gamma\tau\mL_\rmu)}\vphi(0)$.
% It has a solution $\vf_1(t) = e^{-(\kappa_G \mL_{\rmd} +\kappa_G \mL_{\rmu})t } \vf_1(0)$. 
% The diffusion term on the right says the change of $f_1(t)$ on an edge $e$ comes from its differences with down neighbors and up neighbors, which are weighted differently.  
This equation describes the diffusion process on the edge space of $\textSC$ that was used for network analysis \citep{muhammadControlUsingHigher2006,devilleConsensusSimplicialComplexes2021}, often arising as the limit of random walks on edges \citep{schaubRandomWalksSimplicial2020}. 
The covariance $\mK_1$ within this context encodes the proportion of edge flow traveling from edge $e$ to $e^\prime$ via down and up edge adjacencies.
Its vector field counterpart was used for shape analysis \citep{zobelGeneralizedHeatKernel,sharpVectorHeatMethod2020}.
% While both edge and graph (node) diffusion processes converge towards harmonic states as $t\to\infty$, the harmonic state of the former can be non-constant, as opposed to that of the latter, as long as the graph is connected, as shown in \cref{fig:diffusion_demonstration}.  
Compared to the graph (node) diffusion converging ($t\to\infty$) to the state that is constant on all nodes as long as the graph is connected, the harmonic state of the edge diffusion can be non-constant, but lies in the span of $\mU_H$. We refer to \cref{app:diffusion_on_edges} for visualizations of the two harmonic states.

\paragraph{Complexity}
The kernels of HC edge GPs can be constructed in a scalable way by considering the $l$ largest eigenvalues with off-the-shelf eigen-solvers, e.g., Lanczos algorithm. We refer to \cref{app-subsec:complexity} for more details on the complexity of implementing HC GPs, as well as sampling from them.

\subsection{Node-Edge-Triangle GP Interactions} \label{sec:alternative_constuction}
%\vspace{-1mm}
The gradient and curl components of edge functions are (co)derivatives of some node and triangle functions, specifically, $\vf_G=\mB_1^\top\vf_0$ and $\vf_C=\mB_2\vf_2$ as in \cref{eq.hodge-composition-expression}. 
Since the derivative of a GP is also a GP, we can then construct a gradient GP from node GPs. 
\begin{corollary} \label{cor:gradient-of-node-gp}
  Suppose a node function $\vf_0$ is a GP $\vf_0\sim\gG\gP(\vzero,\mK_0)$ with $\mK_0=\Psi_0(\mL_0)=\mU_0 \Psi_0(\mLambda_0)\mU_0^\top$. 
  Then, its gradient is an edge GP $\vf_G\sim\gG\gP(\vzero,\mK_G)$ where $\mK_G = \mB_1^\top\mK_0\mB_1 = \mU_G\Psi_G(\mLambda_G)\mU_G^\top$ with 
  \begin{equation}
    \Psi_G(\mLambda_G) = \mLambda_G \Psi_0(\mLambda_G).%\vspace{-1mm}
  \end{equation}
\end{corollary}
%\vspace{-2mm}
The proof follows from (i) derivatives preserving Gaussianity, and (ii) $\mL_0$ and $\mL_{\rmd}$ having the same nonzero eigenvalues.
We can also obtain a curl edge GP from a GP on triangles likewise. 
In turn, for an edge GP, its div is a node GP and its curl is a GP on triangles.
We refer to \cref{app:interaction-node-edge} for the proof and  more details.

\begin{figure*}[t!]
  % \hspace{-20pt}
  %\vspace{-4mm}
  \makebox[\textwidth][c]{
  \hspace{-10pt}
  \begin{subfigure}{0.2\linewidth}
    \includegraphics[width=1.07\linewidth]{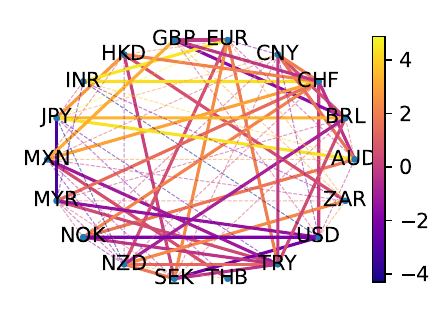} 
    \caption{Ground truth}
  \end{subfigure}
  \begin{subfigure}{0.2\linewidth}
    \includegraphics[width=1.07\linewidth]{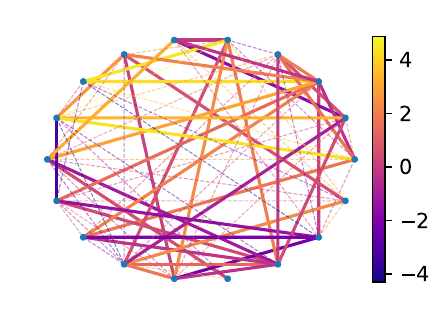} 
    \caption{Mean, HC \Matern}
  \end{subfigure}
  \begin{subfigure}{0.2\linewidth}
    \includegraphics[width=1.07\linewidth]{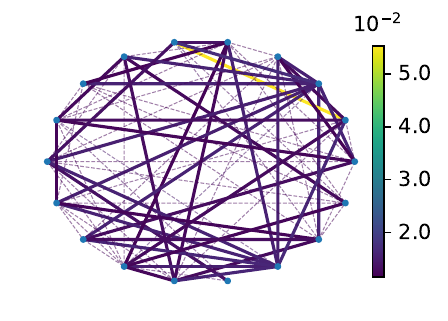} 
    \caption{Standard deviation}
  \end{subfigure}
  \begin{subfigure}{0.2\linewidth}
    \includegraphics[width=1.07\linewidth]{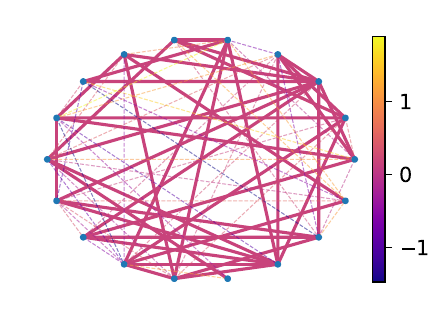} 
    \caption{Mean, Non-HC}
  \end{subfigure}
  \begin{subfigure}{0.2\linewidth}
    %\vspace{2mm}
    \includegraphics[width=1.03\linewidth]{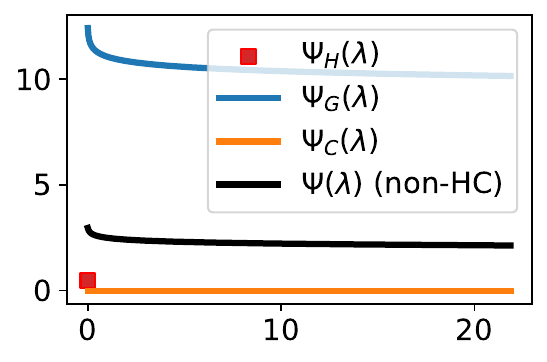} 
    % %\vspace{0.2mm}
    \caption{Learned kernels}
    \label{fig-forex:spectral_learned_kernel}
  \end{subfigure}
  }
  %\vspace{-6mm}
  \caption{(a-d) Interpolating a smaller forex market (for better visibility) with train ratio $50\%$ where dashed (solid) edges are used for training (testing).
  (e) Learned \Matern~kernels in the spectrum of the Hodge Laplacian, $\Psi(\lambda)$ for non-HC GP and $\Psi_\Box(\lambda)$ with $\Box=\{H,G,C\}$ for HC GPs.}
  \label{fig:visualization_forex}
  % \vspace{-3mm}
\end{figure*}

Exploiting this interaction between GPs on nodes, edges and triangles can lead to new useful GPs, especially when functions on nodes, edges and triangles are intrinsically related by physical laws. 
For example, in water networks, water flowrates in pipes are often related to the gradient of hydraulic heads on nodes, as we will show in \cref{subsec:wsn_experiments}. 
This implies that given an appropriate node GP, say, node \Matern~GP in \cref{eq.graph-matern-gp}, a good edge GP prior can be imposed as its gradient, as in \cref{cor:gradient-of-node-gp}.
% Another example is the neurontransmitter current induced by the neuron potentials on nodes of brain networks. 
% On the other hand, the user demand on nodes is the div of flowrates, implying a good GP for modeling the demand can be constructed from the edge GP for flowrates.
% For $\vf_1$ in \cref{def:hodge-compositional-gp}, we have 
% \begin{equation}
%     \textDiv \, \vf_1 \sim \gG\gP(\vzero, \mB_1\mK_G\mB_1^\top) 
%     \textCurl \,\vf_1 \sim \gG\gP(\vzero, \mB_2^\top\mK_C\mB_2)
% \end{equation}
% Thus, we can define edge GPs from node and/or triangle GPs because the gradient and curl parts of edge functions come from the (co)derivative of some node and triangle functions as in \cref{eq.hodge-composition-expression}.
Furthermore, by leveraging this interaction, we can construct HC edge GPs as follows. 
\begin{proposition} \label{prop:alternative-construction}
  Let $\vf_1$ be an edge function defined in \cref{eq.hodge-composition-expression} with harmonic component $\vf_H$, node function $\vf_0$ and triangle function $\vf_2$. 
  If we model $\vf_0$ as a GP on nodes $\vf_0 \sim \gG\gP(\vzero, \mK_0)$, model $\vf_2$ as a GP on triangles $\vf_2 \sim \gG\gP(\vzero, \mK_2)$, and $\vf_H$ as a harmonic GP $\vf_H\sim \gG\gP(\vzero,\mK_H)$, then we have GP $\vf_1 \sim \gG\gP(\vzero, \mK_1)$ with  
  \begin{equation}\label{eq.alternative-kernel}
    \mK_1 = \mK_H + \mB_1^\top\mK_0\mB_1 + \mB_2 \mK_2 \mB_2^\top.
  \end{equation}
\end{proposition}
%\vspace{-2mm}
See proof in \cref{app:alternative_gp}. 
This alternative HC GP incorporates the Hodge theorem prior in a way that directly relates the node potential and the triangle function. 
It can be applicable when GP priors of node or triangle functions are more discernible.
Similar ideas for general cellular complexes are studied in the concurrent paper by \citet{alain2023}.

\paragraph{Continuous Counterparts} Edge functions can be viewed as discrete analogs of vector fields.
\citet{berlinghieri2023gaussian} studied the models similar to our HC edge GPs in \cref{eq.alternative-kernel} for Euclidean vector fields and the concurrent work by \citet{robertnicoud2024} studies similar models for vector fields on manifolds.

% They are more applicable when the underlying physical relationships exist between the corresponding functions and the GP priors on the original simplices are easier to construct.
% A continuous analogy has been applied by \citet{berlinghieri2023gaussian} to construct Helmholtz GPs for vector fields. 
% For $\vf_1$ in \cref{def:hodge-compositional-gp}, we have 
% \begin{equation}
%     \textDiv \, \vf_1 \sim \gG\gP(\vzero, \mB_1\mK_G\mB_1^\top) 
%     \textCurl \,\vf_1 \sim \gG\gP(\vzero, \mB_2^\top\mK_C\mB_2)
% \end{equation}
% Moreover, we show the gradient and curl kernels of GPs in \cref{prop:alternative-construction}.

%\vspace{-1mm}
\section{EXPERIMENTS}\label{sec:experiments}
%\vspace{-1mm}
We apply HC GPs for edge-based inference tasks in three applications: foreign currency exchange (forex), ocean current and water supply networks (WSNs). 
We showcase the structured prior on edges in these tasks by comparing them to baselines: 
(i) Euclidean GPs with RBF and \Matern~kernels, and
(ii) Node GPs on the line-graph---built by exchanging the nodes with edges in the original graph \citep{godsilAlgebraicGraphTheory2001}.
To highlight the prior of the Hodge decomposition, we also compare with non-HC GPs.
For each of them, we consider \Matern~and diffusion kernels.
We perform GP regression with Gaussian likelihood for model fitting using the \texttt{GPyTorch} framework \citep{gardner2018gpytorch}. 
We use the root mean squared error (RMSE) to evaluate the predictive mean and the negative log predictive density (NLPD) for prediction uncertainty. 
We refer to \cref{app:experiment_details} for full experimental details. 
% Our code is available at: \url{https://github.com/cookbook-ms/Hodge-Edge-GP}.

\begin{figure*}[t!]
  %\vspace{0mm}
  \begin{subfigure}{0.33\linewidth}
    \includegraphics[width=1\linewidth]{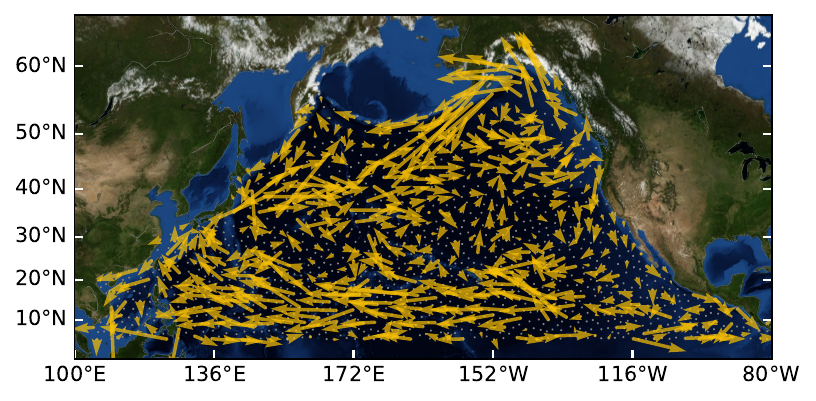}
    %\vspace{-5mm}
    \caption{Original ocean current}
  \end{subfigure}
  \begin{subfigure}{0.33\linewidth}
    \includegraphics[width=1\linewidth]{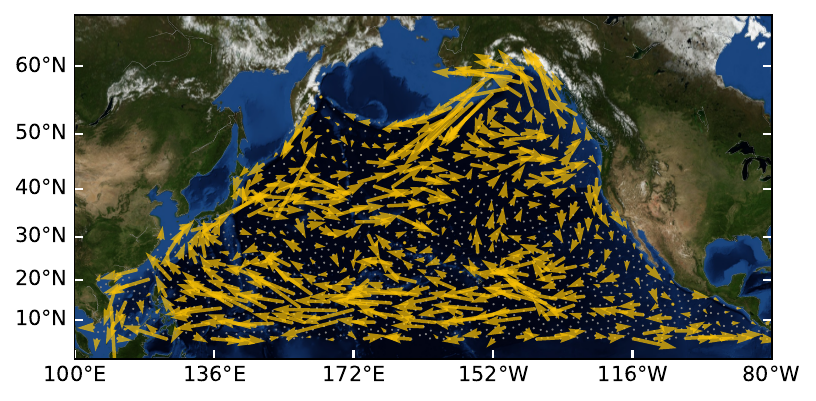}
    %\vspace{-5mm}
    \caption{Posterior mean}
    \label{fig-ocean:mean}
  \end{subfigure}
  % \begin{subfigure}{0.33\linewidth}
  %   \includegraphics[width=1\linewidth]{figures/ocean_flow/std_vector_field.pdf}
  %   %\vspace{-5mm}
  %   \caption{Standard deviation}
  %   \label{fig-ocean:std}
  % \end{subfigure}
\begin{subfigure}{0.33\linewidth}
    \includegraphics[width=1.08\linewidth]{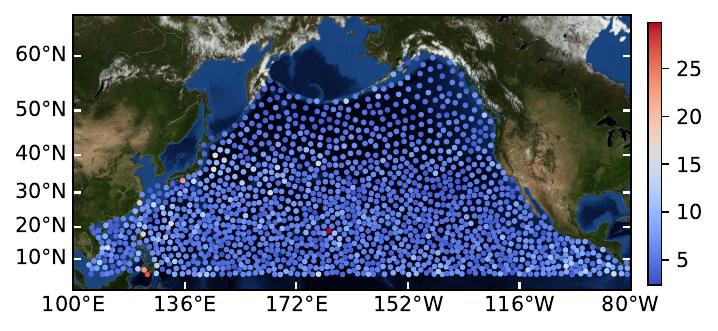}
    %\vspace{-5mm}
    \caption{Standard deviation}
    \label{fig-ocean:std_approx}
  \end{subfigure}
  \begin{subfigure}{0.33\linewidth}
    \includegraphics[width=1\linewidth]{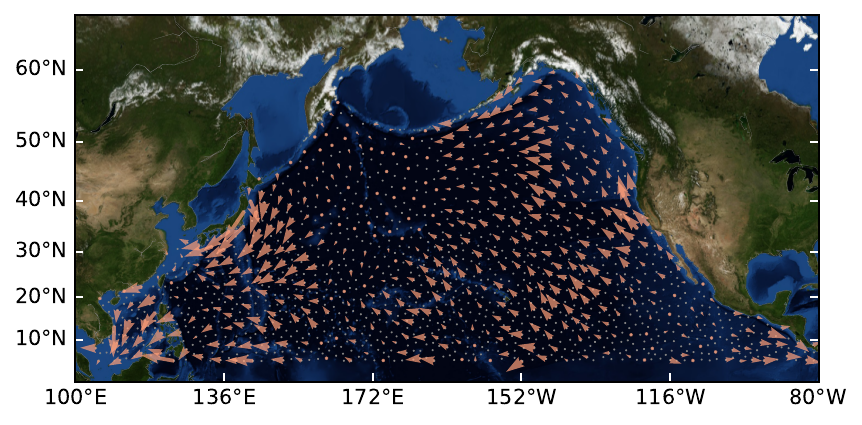}
    \caption{Curl-free component}
    \label{fig-ocean:gradient_sample}
  \end{subfigure}
  \begin{subfigure}{0.33\linewidth}
    \includegraphics[width=1\linewidth]{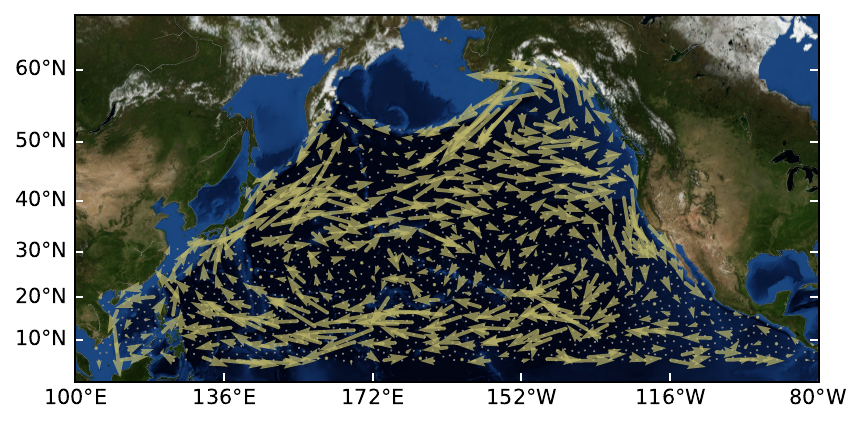}
    \caption{Div-free component}
    \label{fig-ocean:curl_sample}
  \end{subfigure}
  \begin{subfigure}{0.33\linewidth}
    \vspace{-1mm}
    \includegraphics[width=0.95\linewidth]{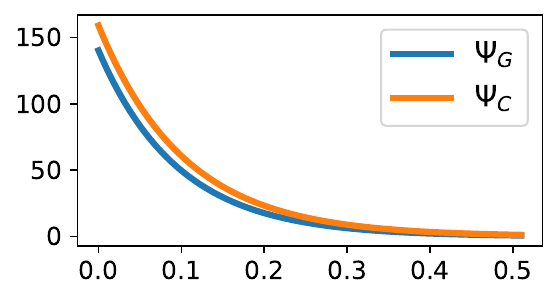}
    \vspace{-1mm}
    \caption{Learned HC kernel}
    \label{fig-ocean:learned_kernel}
  \end{subfigure}
  %\vspace{-6mm}
  \caption{(a-b) Ground truth and interpolated ocean current in the vector field domain. (c) Standard deviation approximated by sampling 50 edge flows from the predictive posterior distribution and converted to the vector field domain. (d-e) The curl-free and div-free components directly obtained from the learned kernels. (f) Learned diffusion kernels $\Psi_G(\lambda)$ and $\Psi_C(\lambda)$ of the HC GP in the spectrum of the Hodge Laplacian.}
  \label{fig:ocean_flow_illustration}
  %\vspace{-3mm}
\end{figure*}

%\vspace{-1mm}
\subsection{Foreign Currency Exchange} \label{subsec:forex_experiments}
%\vspace{-1mm}
A forex market can be modeled as a network where nodes represent currencies and edges the exchangeable pairs \citep{jiangStatisticalRankingCombinatorial2011}. 
Forex rates in a fair market ideally satisfy the \emph{arbitrage-free} condition: for any currencies $i,j,k$, we have $r^{i/j}r^{j/k} = r^{i/k}$ with $r^{i/j}$ the rate between $i$ and $j$. 
That is, the exchange path $i\to j\to k$ provides no gain or loss over a direct path $i\to k$. 
If we model forex rates as edge flows $f_1(i,j) = \log (r^{i/j})$, this condition can be translated into that $\vf_1$ is a gradient flow, being curl-free, i.e., $f_1(i,j) + f_1(j,k) - f_1(i,k)=0$. 
Here we consider real-world forex data on 2018/10/05 with 25 most traded currencies forming 210 exchangeable pairs and 710 triangles, formed by any three pairwise exchangeable currencies \citep{forex_data,jiaGraphbasedSemiSupervisedActive2019}. 
We randomly sample 20\% of edges for training and test on the rest. 
\begin{table}[!t]
   % \vspace{-1mm}
  \centering
  \caption{Forex rates inference results.}
  %\vspace{-3mm}
  \resizebox{1\linewidth}{!}{
    \begin{tabular}{lllll}
      \toprule
      \multirow{2}{*}{Method}  & \multicolumn{2}{c}{RMSE } & \multicolumn{2}{c}{NLPD}  \\ 
      \cmidrule(lr){2-3}  \cmidrule(lr){4-5} & Diffusion & \Matern~ &  Diffusion & \Matern~\\
      \midrule      
      Euclidean & $2.17\pm0.13$ & $2.19\pm 0.12$ & $2.12\pm0.07$ & $2.20\pm0.18$\\
      Line-Graph & $2.43\pm 0.07$ & $2.46\pm 0.07$ & $2.28\pm0.04$ & $2.32\pm 0.03$ \\ 
      Non-HC & $2.48\pm0.07$ & $2.47\pm 0.08$ & $2.36\pm 0.07$ & $2.34\pm 0.04$ \\
      HC & $0.08\pm 0.12$ & $0.06\pm 0.12$ & $-3.52\pm0.02$ & $-3.52\pm 0.02$ \\ 
      \bottomrule
      \end{tabular}
  }
  \label{tab:results_1}
  \vspace{0mm}
\end{table}

From \cref{tab:results_1}, we see that HC GPs achieve significantly lower RMSEs with high certainty (small NLPDs), as visualized in \cref{fig:visualization_forex}.
This shows their ability to automatically capture the curl-free nature of the forex rates.
% , resulting from their separate learning of the three Hodge GPs.
As shown in \cref{fig-forex:spectral_learned_kernel}, the HC \Matern~GP learns that harmonic and curl components should vanish. 
In contrast, the other three give poor predictions, due to:
(i) Euclidean GPs being oblivious of the structure of edge functions; 
(ii) line-graph GPs imposing inappropriate structure through node priors; 
and 
(iii) non-HC GPs being unable to induce the curl-free prior without removing the gradient. 
This results from sharing parameters in their kernels for different Hodge components.
As shown in \cref{fig-forex:spectral_learned_kernel}, the non-HC \Matern~GP learns a nonzero kernel in the whole spectrum, unable to remove the non-arbitrage-free part.

%\vspace{-3mm}
\subsection{Ocean Current Analysis}
%\vspace{-2mm}
% The Hodge Laplacian based approaches were applied to ocean current velocity learning by \citet{chenHelmholtzianEigenmapTopological2021} with better efficiency compared to the traditional one \citep{berry2020spectral}. 
% They constructed a $\rm{SC}_2$ of 1500 buoys sampled from North Pacific ocean current records in 2010-2019 \citep{lumpkin2019global} and converted the velocity fields to edge flows using the linear integration approximation \citep[Appendix H]{chenHelmholtzianEigenmapTopological2021}. 
We consider the edge-based ocean current learning following the setup in \citet{chenHelmholtzianEigenmapTopological2021}. 
The ocean current velocity fields were converted using the linear integration approximation to edge flows within a $\textSC$ whose nodes are 1500 buoys sampled from North Pacific ocean drifter records in 2010-2019 \citep{lumpkin2019global}.
We apply both non-HC and HC GPs to predict the converted edge flows.
Given the large number of edges ($\sim$20k), we consider a truncated approximation of kernels with eigenpairs associated with the 500 largest eigenvalues \citep{knyazev2001toward}. 
% We refer to \citet{chenHelmholtzianEigenmapTopological2021} for more details on the dataset.
We randomly sample 20\% of edges for training and test on the rest. 

\begin{table}[t!]
  \centering
  \caption{Ocean current inference results.}
  %\vspace{-3mm}
  \label{tab:ocean_flow}
  \resizebox{1\linewidth}{!}{
    \begin{tabular}{lllll}
      \toprule
      \multirow{2}{*}{Method}  & \multicolumn{2}{c}{RMSE } & \multicolumn{2}{c}{NLPD}  \\ 
      \cmidrule(lr){2-3} \cmidrule(lr){4-5}  & Diffusion & \Matern~ &  Diffusion & \Matern~\\
      \midrule      
      Euclidean & $1.00\pm0.01$ & $1.00\pm0.00$ & $1.42\pm0.01$ & $1.42\pm0.10$\\
      Line-Graph & $0.99\pm0.00$ & $0.99\pm0.00$ & $1.41\pm0.00$ & $1.41\pm0.00$ \\ 
      Non-HC & $0.35\pm0.00$ & $0.35\pm 0.00$ & $0.33\pm 0.00$ & $0.36\pm 0.03$ \\
      HC & $0.34\pm 0.00$ & $0.35\pm 0.00$ & $0.33\pm0.01$ & $0.37\pm 0.04$ \\ 
      \bottomrule
      \end{tabular}
  }
  \vspace{-1mm}
\end{table}

From \cref{tab:ocean_flow}, we see that HC and non-HC GPs exhibit similar performance. 
This arises from the comparable behavior of the gradient and curl components, as depicted in \cref{fig-ocean:learned_kernel}, where the learned gradient and curl diffusion kernels display close patterns. 
In contrast, Euclidean and line-graph GPs give poor predictions emphasizing the importance of structured edge priors. 

We further convert the predicted edge flows into the vector field domain, as shown in \cref{fig-ocean:mean}, based on \citet{chenHelmholtzianEigenmapTopological2021}.
We see that the predictions capture the pattern of the original velocity field.
We approximate the predicted velocity field uncertainty by computing the average $\ell_2$ distance per location from 50 posterior samples to the mean in the vector field domain. 
As shown in \cref{fig-ocean:std_approx}, the standard deviation estimated this way is small in most locations except for a few exceptions (small islands at the bottom left) where the original field is more discontinuous. 
% is mainly localized around the boundaries where the velocity fields exhibit more discontinuities.
Moreover, since HC GPs enable the direct recovery of gradient and curl components, we show their corresponding vector fields in \cref{fig-ocean:gradient_sample,fig-ocean:curl_sample}, giving better insights into how ocean currents behave, of particular interest in oceanography.
For example, we can observe the well-known North Pacific gyres including the North Equatorial, Kuroshio and Alaska currents in \cref{fig-ocean:curl_sample}.

% with reference 

\begin{figure*}[t!]
  % \hspace{-20pt}
  % \vspace{-3mm}
  \makebox[\textwidth][c]{
  \begin{subfigure}{0.33\linewidth}
    \includegraphics[width=1\linewidth]{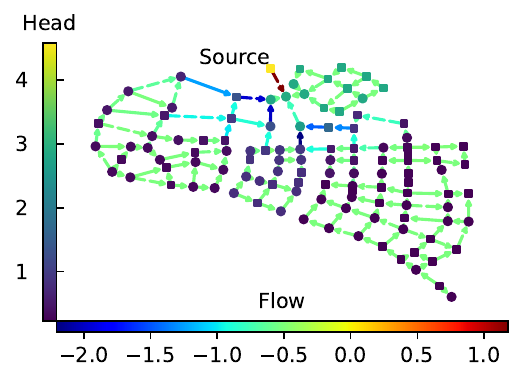} 
    \caption{Ground truth}
    \label{fig-wsn:groundtruth}
  \end{subfigure}
  \begin{subfigure}{0.33\linewidth}
    \includegraphics[width=1\linewidth]{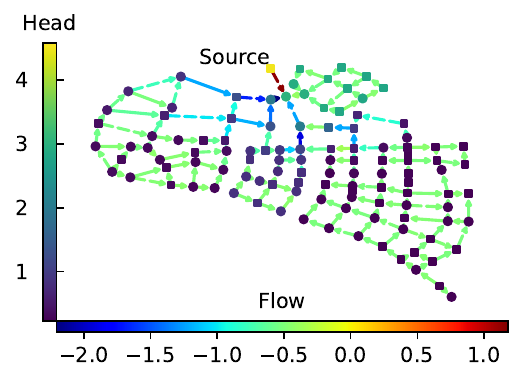} 
    \caption{Mean, HC \Matern}
    \label{fig-wsn:hc-mean}
  \end{subfigure}
  \begin{subfigure}{0.33\linewidth}
    \includegraphics[width=1\linewidth]{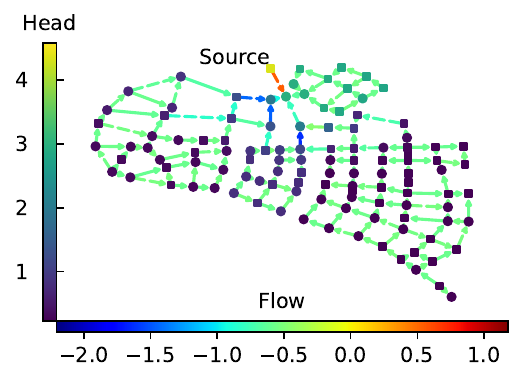} 
    \caption{Mean, non-HC \Matern}
    \label{fig-wsn:non-hc-mean}
  \end{subfigure}
  }
  \makebox[\textwidth][c]{
  %\vspace{-3mm}
  \begin{subfigure}{0.33\linewidth}
    \includegraphics[width=1.05\linewidth]{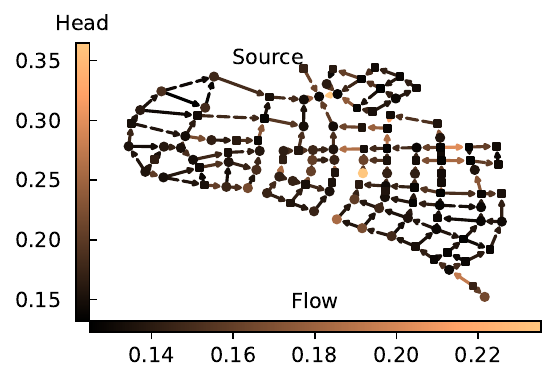} 
    \caption{Std, HC \Matern}
    \label{fig-wsn:hc-std}
  \end{subfigure}
  \begin{subfigure}{0.33\linewidth}
    \includegraphics[width=1.05\linewidth]{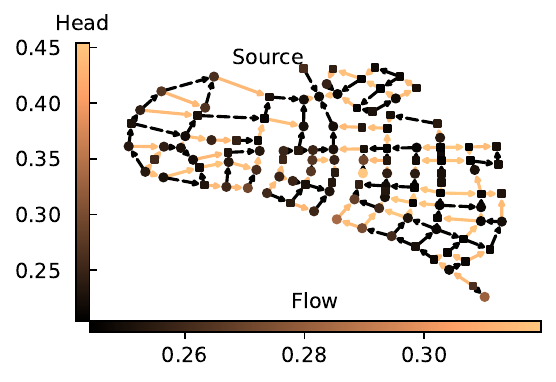} 
    \caption{Std, non-HC \Matern}
    \label{fig-wsn:non-hc-std}
  \end{subfigure}
  \begin{subfigure}{0.33\linewidth}
    \includegraphics[width=0.96\linewidth]{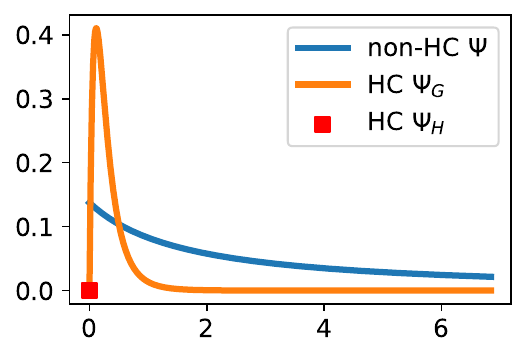} 
    \caption{Learned \Matern~edge kernels} 
    \label{fig-water_network:spectral_kernel}
  \end{subfigure}
  }
  %\vspace{-6mm}
  \caption{(a-e) Posterior mean and standard deviation (std) based on the \Matern~node GPs, and the HC and non-HC \Matern~edge GPs. Squared (Circled) nodes represent the node samples for training (testing). Dashed (solid) edges denote the edge samples for training (testing). (f) The learned edge GP kernels in the spectrum, $\Psi(\lambda)$ for non-HC GP and $\Psi_{H}(0), \Psi(\lambda)$ for HC GPs.}
  %\vspace{-3mm}
\end{figure*}

%\vspace{-2mm}
\subsection{Water Supply Networks} 
%\vspace{-1mm}
\label{subsec:wsn_experiments}
Network-based methods have been used in WSNs where tanks or reservoirs are represented by nodes, and pipes by edges \citep{zhou2022bridging}. 
By modeling the hydraulic heads as node functions $\vf_0$ and the water flowrates as edge functions $\vf_1$, the commonly used empirical equation connecting the two reads as $\mB_1^\top\vf_0 = \bar{\vf}_1:= \diag(\vr)\vf_1^{1.852}$ where $r_e$ is the resistance of pipe $e$ and the exponentiation is applied element-wise \citep{dini2014new}. 

\begin{table}[t!]
  \centering
  \caption{WSN inference results.}
  \label{tab:water_network}
  %\vspace{-3mm}
  \resizebox{1.02\linewidth}{!}{
  \begin{tabular}{lllll}
  \toprule
  \multirow{2}{*}{Method}  & \multicolumn{2}{c}{Node Heads} & \multicolumn{2}{c}{Edge Flowrates} \\
  \cmidrule(lr){2-3} \cmidrule(lr){4-5}   & RMSE  & NLPD & RMSE & NLPD \\
  \midrule
  Diffusion, non-HC & $0.16\pm0.05$ & $0.72\pm2.06$ & $0.32\pm 0.05$  & $0.97\pm1.80$ \\
  \Matern, non-HC  & $0.16\pm 0.04$  & $0.71\pm2.39$ & $0.26\pm 0.05$  & $0.10\pm0.13$ \\ 
  \midrule
  Diffusion, HC  & $0.15\pm0.04$ & $-0.47\pm0.14$ & $0.22\pm0.03$  & $-0.20\pm0.13$\\
  \Matern, HC  & $0.15\pm 0.04$ & $-0.25\pm0.48$ & $0.23\pm 0.03$ & $-0.45\pm0.49$ \\ 
  \bottomrule
  \end{tabular}
  }
  \vspace{-1mm}
\end{table}
We consider the Zhi Jiang WSN with 114 tanks (including one source) and 164 pipes (without triangles, \citet{dandy06}) and simulate a scenario based on \citet{klise2017water}. 
We perform joint state estimation of heads $\vf_0$ and the adjusted flowrates $\bar{\vf}_1$, by modeling them as GPs on nodes and edges, respectively. 
To compare HC and non-HC edge GPs, for a node GP with kernel $\mK_0$, we consider the HC GP as its gradient, as discussed in \cref{cor:gradient-of-node-gp}. 
For the non-HC one, we consider a kernel $\mK_1$ of the same type as $\mK_0$.
We randomly sample 50\% of nodes and edges for training and test on the rest. 
% We evaluate the predictions on nodes and edges separately.

From \cref{tab:water_network}, we see that while the mean predictions of heads remain similar whether we use HC or non-HC edge GPs, the former perform better for edge flows, particularly in the pipes around the source, as shown in \cref{fig-wsn:non-hc-mean,fig-wsn:hc-mean}. 
Moreover, HC GPs have better prediction uncertainty with smaller average NLPDs for both heads and flowrates, as illustrated in \cref{fig-wsn:hc-std,fig-wsn:non-hc-std}. 
This is because HC GPs that we use share parameters with node GPs, helping to calibrate the uncertainty of head predictions. 
They also capture the physical prior of the pipe equation that assumes flowrates are a gradient flow.
As shown in \cref{fig-water_network:spectral_kernel}, the HC \Matern~GP learns a kernel with a trivial harmonic prior and a nonzero gradient prior in small eigenvalues, reflecting the gradient nature of the pipe flowrates. 
Note that due to the randomness of training samples, the WSN, having small edge connectivity, may become disconnected, causing the significant variance in NLPDs.

% \vspace{-2mm}
\section{CONCLUSION}
% \vspace{-2mm}
% We introduced Hodge-compositional (HC) Gaussian processes, a principled framework for modeling functions over the edges of simplicial 2-complexes.
% We build them as the combination of three GPs---each with individual kernel hyperparameters---on the gradient, curl and harmonic parts of edge functions given by the Hodge decomposition.
% Each of the three can be linked to their respective SPDE representation on edges. 
% Compared to the non-HC edge GPs, directly extended from graph GPs, HC GPs benefit from the separate learning of each Hodge part, having better expressivity and interpretablity.
% We showed the practical potentials of the HC GPs in statistical learning of flow-type data related to currency exchange rates, ocean currents and water supply networks. 

We introduced Hodge-compositional (HC) Gaussian processes (GPs) for modeling functions on the edges of simplicial 2-complexes. 
These HC GPs are constructed by combining three individual GPs, each designed to capture the gradient, curl and harmonic components of the Hodge decomposition of edge functions. 
This allows them to learn each component separately, making them more expressive and interpretable when compared to alternatives.
Moreover, they can be extended to jointly model signals on nodes, edges and faces of the domain.
% They can also be constructed by leveraging the physical interactions between functions on nodes, edges and triangles.
We demonstrated their practical potential in learning real-world flow data.
% We also showed they admit an alternative construction that exploits the interactions between functions on nodes, edges and triangles. 
% We hope our work inspires further advancements in the field of non-Euclidean GPs.
% This principled approach into statistical learning in complex network domains.

% Our setting has discrete networks as the underlying domain, often arising in modern data applications. 
% While we used some analogies in our discussions to continuous vector fields, we do not dwell on the formal relationships between discrete edge GPs and continuous ones, noting that \citet{sanz2022spde} studied how graph \Matern~kernels converge to their Riemannian counterparts \citep{borovitskiy2020matern,borovitskiyMatErnGaussian2021}.
% In future work, it would be possible to improve the scalability of edge GPs and add the temporal dimension for temporal-edge data inference. 

%%%%%%%%%%%%%%%%%%%%%%%%%%%%%%%%%%%%%%%%%%%%%%%%%%%%%%%%%%%%
\subsubsection*{Acknowledgements}
MY is supported by the TU Delft AI Labs Programme and
VB by an ETH Zürich Postdoctoral Fellowship. 

%%%%%%%%%%%%%%%%%%%%%%%%%%%%%%%%%%%%%%%%%%%%%%%%%%%%%%%%%%%%
\bibliography{MyLibrary.bib}

%%%%%%%%%%%%%%%%%%%%%%%%%%%%%%%%%%%%%%%%%%%%%%%%%%%%%%%%%%%%
\newpage
\section*{Checklist}

\begin{enumerate}
  \item For all models and algorithms presented, check if you include:
  \begin{enumerate}
    \item A clear description of the mathematical setting, assumptions, algorithm, and/or model. [Yes] We include the problem setting for general Gaussian process in Background, which holds the same for our problem of edge Gaussian process.  
    \item An analysis of the properties and complexity (time, space, sample size) of any algorithm. [Yes] We provide the complexity about the edge kernel in Appendix.
    \item (Optional) Anonymized source code, with specification of all dependencies, including external libraries. 
    [Yes] 
  \end{enumerate}

  \item For any theoretical claim, check if you include:
  \begin{enumerate}
    \item Statements of the full set of assumptions of all theoretical results. [Yes]
    \item Complete proofs of all theoretical results. [Yes] We provide all proofs and derivations in Appendix.
    \item Clear explanations of any assumptions. [Yes]     
  \end{enumerate}

  \item For all figures and tables that present empirical results, check if you include:
  \begin{enumerate}
    \item The code, data, and instructions needed to reproduce the main experimental results (either in the supplemental material or as a URL). [Yes]
    \item All the training details (e.g., data splits, hyperparameters, how they were chosen). [Yes]
    \item A clear definition of the specific measure or statistics and error bars (e.g., with respect to the random seed after running experiments multiple times). [Yes] All measures are defined and experiments are run for ten times. 
    \item A description of the computing infrastructure used. (e.g., type of GPUs, internal cluster, or cloud provider). [Yes]
  \end{enumerate}
 
  \item If you are using existing assets (e.g., code, data, models) or curating/releasing new assets, check if you include:
  \begin{enumerate}
    \item Citations of the creator If your work uses existing assets. [Yes] We cite the works on all the dataset we use. 
    \item The license information of the assets, if applicable. [Not Applicable]
    \item New assets either in the supplemental material or as a URL, if applicable. [Not Applicable]
    \item Information about consent from data providers/curators. [Not Applicable]
    \item Discussion of sensible content if applicable, e.g., personally identifiable information or offensive content. [Not Applicable]
  \end{enumerate}
 
  \item If you used crowdsourcing or conducted research with human subjects, check if you include:
  \begin{enumerate}
    \item The full text of instructions given to participants and screenshots. [Not Applicable]
    \item Descriptions of potential participant risks, with links to Institutional Review Board (IRB) approvals if applicable. [Not Applicable]
    \item The estimated hourly wage paid to participants and the total amount spent on participant compensation. [Not Applicable]
  \end{enumerate}
 
  \end{enumerate}

%%%%%%%%%%%%%%%%%%%%%%%%%%%%%%%%%%%%%%%%%%%%%%%%%%
\newpage
\onecolumn
\appendix\newpage\markboth{Appendix}{Appendix}
\renewcommand{\thesection}{\Alph{section}}
\numberwithin{equation}{section}
\numberwithin{theorem}{section}
\numberwithin{figure}{section}
\numberwithin{table}{section}
\hsize\textwidth
\linewidth\hsize 
\toptitlebar {\centering
{\Large\bfseries {Supplementary Materials for \\ Hodge-Compositional Edge Gaussian Processes} \par}}
\bottomtitlebar 
% \section*{\LARGE\centering{Supplementary Material for \\ ``Hodge-Aware Learning on Simplicial Complexes''}}
% \title{Supplementary Materials}
% \setcounter{section}{0}

\section{BACKGROUND} \label{sc-basics}

\subsection{Algebraic Representation of Simplicial 2-Complexes}
For a $\textSC$ with $N_0$ nodes, $N_1$ edges and $N_2$ triangles in \cref{sec:background}, the entries of $\mB_1\in\R^{N_0\times N_1}$ and $\mB_2\in\R^{N_1\times N_2}$ are given by  
\begin{equation}
  [\mB_1]_{ie} = 
  \begin{cases}
    -1, &\text{for } e=[i,\cdot]\\
    1, &\text{for } e=[\cdot,i] \\
    0, &\text{otherwise. }
  \end{cases}
  \quad 
  [\mB_2]_{et} = 
  \begin{cases}
    1, &\text{for } e=[i,j], t=[i,j,k]\\
    -1, &\text{for } e=[i,k], t=[i,j,k]\\
    1, &\text{for } e=[j,k], t=[i,j,k]\\
    0, &\text{otherwise. }
  \end{cases}
\end{equation}
In the following, we show the two incidence matrices of the SC in \cref{fig:sc_example}. 
\begin{equation}
  \mB_1 = 
  \kbordermatrix{
      & e_1 & e_2 & e_3 & e_4 & e_5 & e_6 & e_7 & e_8 & e_9 & e_{10} \\
    1 & -1 & -1 & -1 & 0 & 0 & 0 & 0 & 0 & 0 & 0 \\
    2 & 1 & 0 & 0 & -1 & -1 & 0 & 0 & 0 & 0 & 0 \\
    3 & 0 & 1 & 0 & 1 & 0 & -1 & -1 & -1 & 0 & 0 \\
    4 & 0 & 0 & 1 & 0 & 0 & 1 & 0 & 0 & 0 & 0 \\
    5 & 0 & 0 & 0 & 0 & 1 & 0 & 1 & 0 & -1 & -1 \\
    6 & 0 & 0 & 0 & 0 & 0 & 0 & 0 & 1 & 1 & 0 \\
    7 & 0 & 0 & 0 & 0 & 0 & 0 & 0 & 0 & 0 & 1 \\
  }, \quad 
  \mB_2 = 
  \kbordermatrix{
      & t_1 & t_2 & t_3 \\
    e_1 & 1 & 0 & 0 \\
    e_2 & -1 & 0 & 0 \\
    e_3 & 0 & 0 & 0 \\
    e_4 & 1 & 1 & 0 \\
    e_5 & 0 & -1 & 0 \\
    e_6 & 0 & 0 & 0 \\
    e_7 & 0 & 1 & 1 \\
    e_8 & 0 & 0 & -1 \\
    e_9 & 0 & 0 & 1 \\
    e_{10} & 0 & 0 & 0 \\
  }
\end{equation}

\section{EDGE GAUSSIAN PROCESSES}
Here we provide the additional details on \cref{sec:edge-gp} and the missing proofs.

\subsection{Derivation of Edge GPs from SPDEs on Edges}
Here we derive the edge \Matern~and diffusion GPs in \cref{eq.simple_matern_diffusion} from the two SPDEs in \cref{eq.spde-edge}. 
\begin{proposition} \label{app:prop-simple-spde-general}
  Given the SPDE with a general differential operator $\Phi(\mL_1)=\mU_1\Phi(\mLambda_1)\mU_1^\top$ and the stochastic Gaussian noise process $\vw_1\sim\gN(\vzero,\mI)$
\begin{equation}
  \Phi(\mL_1)\vf_1 = \vw_1 ,
\end{equation}
its solution is an edge GP 
\begin{equation}
  \vf_1\sim\gG\gP(\vzero,(\Phi^\top(\mL_1)\Phi(\mL_1))^\dagger)
\end{equation}
\end{proposition}

\begin{proof}
  By writing out its solution
\begin{equation}
  \vf_1 = \Phi^\dagger(\mL_1) \vw_1,
\end{equation}
which is a random process, we can find its covariance as 
\begin{equation}
  \Cov[\vf_1] = \Phi^\dagger(\mL_1) \Cov[\vw_1] (\Phi^\dagger(\mL_1))^\top =   (\Phi^\top(\mL_1)\Phi(\mL_1))^\dagger 
\end{equation}
\end{proof}

\begin{corollary}
  \Matern~and diffusion edge kernels in \cref{eq.simple_matern_diffusion} given as follows 
  \begin{equation}
    \vf_1 \sim \gG\gP \Bigl(\vzero, \Bigl( \frac{2\nu}{\kappa^2}\mI+\mL_1 \Bigr)^{-\nu} \Bigr), \quad 
    \vf_1 \sim \gG\gP \Bigl(\vzero,e^{-\frac{\kappa^2}{2}\mL_1}\Bigr)
  \end{equation}
  are the solutions of the following two SPDEs, respectively.
  \begin{equation}
    \Bigl( \frac{2\nu}{\kappa^2}\mI+\mL_1 \Bigr)^\frac{\nu}{2} \vf_1 = \vw_1, 
    \quad 
    e^{\frac{\kappa^2}{4}\mL_1} \vf_1 = \vw_1.
  \end{equation}
\end{corollary}
\begin{proof}
  By following the procedure in \cref{app:prop-simple-spde-general}, the proof completes. 
\end{proof}

\subsection{Samples of Gradient and Curl Edge GPs}\label{proof:prop-gp-samples}
% \begin{proposition} \label{prop:samples-of-grad-curl-gps-2}
%     Let $\vf_G$ and $\vf_C $ be the gradient and curl Gaussian processes in \cref{eq.grad-curl-gps}, respectively. 
%     Then, their samples are curl-free and divergence-free, respectively. Moreover, samples of a harmonic Gaussian process are both div- and curl-free.
% \end{proposition}
Here we discuss the div and curl properties of samples of gradient and curl GPs in \cref{eq.grad-curl-gps}, which completes the proof of \cref{prop:samples-of-grad-curl-gps} . 
\begin{proposition}
  Consider the gradient and curl GPs
  \begin{equation} \label{app-eq.grad-curl-gps}
    \vf_G \sim \gG\gP(\vzero, \mK_G),  \quad \vf_C \sim \gG\gP(\vzero, \mK_C) 
  \end{equation}
  where the gradient kernel and the curl kernel are 
  \begin{equation}
    \mK_G = \mU_G \Psi_G(\mLambda_G) \mU_G^\top, \quad \mK_C = \mU_C \Psi_C(\mLambda_C) \mU_C^\top. 
  \end{equation}
  Their prior samples are, respectively, curl-free and div-free. 
\end{proposition}
%%%%%%%%%%%%%%%%%%%%%%%%%%%%%%%%%%%%%%%%%%%%%%%%%%%%%%%%%%%%
% Cholesky decomposition
\begin{proof}
  We focus on the case of gradient GPs.
  First, we can decompose the gradient kernel in terms of $\mU_1 = [\mU_H\,\,\mU_G\,\,\mU_C]$ as 
    \begin{equation}
      \mK_G = 
      \mU_1
      \begin{pmatrix}
        \vzero & & \\
         & \Psi_G(\mLambda_G) & \\
        & & \vzero \\ 
      \end{pmatrix}
      \mU_1^\top .
    \end{equation}
    From a vector $\vv=(v_1,\dots,v_{N_1})^\top$ of variables following independent normal distribution, we can draw a random sample of gradient function as 
    \begin{equation}
      \vf_G = \mU_1 \diag([\vzero,\Psi_G^\frac{1}{2}(\mLambda_G),\vzero]) \vv 
    \end{equation}
    where $\diag([\va,\vb,\vc])$ is the diagonal matrix with $(\va,\vb,\vc)^\top$ on its diagonal. 

    Therefore, their curls are 
    \begin{equation}
      \textCurl\, \vf_G = \mB_2^\top \mU_1 \diag([\vzero,\Psi_G^\frac{1}{2}(\mLambda_G),\vzero]) = \mB_2^\top \mU_G \Psi_G^{\frac{1}{2}}(\mLambda_G) = \vzero. 
    \end{equation}
    Likewise, we can show the samples of a curl GP are div-free. 
    \begin{remark}
      % An alternative proof can follow by studying the curl of the gradient GP which is another GP on triangles 
      % \begin{equation}
      %   \textCurl \, \vf_G = \mB_2^\top \vf_G \sim \gG\gP(\vzero,\mB_2^\top\mK_G\mB_2)
      % \end{equation}
      % but the kernel $\mB_2^\top\mK_G\mB_2$ is zero, due to the orthogonality $\mB_2^\top\mU_G = \vzero$. 
      % Thus, the curl of a gradient GP is a zero GP on triangles, as well as its samples. Similarly, one can show the div of a curl GP is a zero GP on nodes, thus, its samples are zero. 
      An alternative proof can follow by studying the curl of the gradient GP which is another GP on triangles as given later by \cref{app-prop:div-curl-gp}.
      The kernel $\mB_2^\top\mK_G\mB_2$ is zero, due to the orthogonality $\mB_2^\top\mU_G = \vzero$. 
      Thus, the curl of a gradient GP is a zero GP on triangles, as well as its samples. Similarly, one can show the div of a curl GP is a zero GP on nodes, thus, its samples are zero. 
    \end{remark}
\end{proof}

\subsection{Derivation of Gradient and Curl GPs from SPDEs} \label{app:derivation-grad-curl-gps}
Here we provide proofs for \cref{prop:spde-grad-curl-gp}, deriving \Matern~and diffusion gradient/curl GPs from their SPDE representations.
\begin{proposition}\label{app-prop:curl-gp-spde}
  Given a scaled curl white noise  $\vw_C\sim\gN(\vzero,\mW_C)$ where $\mW_C = \sigma_C^2\mU_C\mU_C^\top$, consider the following SPDE on edges: 
  \begin{equation} 
    \Phi_C(\mL_{\rmu}) \vf_C = \vw_C, 
  \end{equation}
  with differential operators
  \begin{equation}
      \Phi_C(\mL_{\rmu}) = \Big(\frac{2\nu_C}{\kappa_C^2} \mI + \mL_{\rmu} \Big)^{\frac{\nu_C}{2}}, \quad \Phi_C(\mL_{\rmu})=e^{\frac{\kappa_C^2}{4}\mL_{\rmu}}.
  \end{equation}
  The respective solutions give the curl edge GPs with \Matern~kernel and diffusion kernel
  \begin{equation} 
    \vf_C \sim \gG\gP \Big(\vzero, \sigma_C^2 \mU_C  \Big(\frac{2\nu_C}{\kappa_C^2} \mI + \mL_{\rmu} \Big)^{-\nu_C} \mU_C^\top \Big), \quad 
    \vf_C \sim \gG\gP \Big(\vzero, \sigma_C^2 \mU_C  e^{-\frac{\kappa^2_C}{2}} \mU_C^\top \Big).
  \end{equation}
\end{proposition}

\begin{proof}
 First, consider the \Matern~curl GP case. 
 The corresponding SPDE has the form 
 \begin{eqnarray}
   \Big(\frac{2\nu_C}{\kappa_C^2} \mI + \mL_{\rmu} \Big)^{\frac{\nu_C}{2}} \vf_C = \vw_C, 
 \end{eqnarray}
 with a solution $\vf_C = \Phi_C^\dagger(\mL_{\rmu})\vw_C$.

 Given the scaled curl Gaussian noise process $\vw_C\sim \gG(\vzero, \mW_C)$ with $\mW_C=\sigma_C^2\mU_C\mU_C^\top$, 
 the solution $\vf_C$ is an edge GP following $\vf_C\sim\gG\gP(\vzero,\Cov[\vf_C])$ with the covariance of solution $\vf_C$ as 
 \begin{equation} \label{app-eq:cov-1}
  \Cov[\vf_C] = \Big(\frac{2\nu_C}{\kappa_C^2} \mI + \mL_{\rmu} \Big)^{-\frac{\nu_C}{2}} \mW_C 
  \Big(\frac{2\nu_C}{\kappa_C^2} \mI + \mL_{\rmu} \Big)^{-\frac{\nu_C}{2}}.
 \end{equation}
 Note that we have 
 \begin{equation} \label{app-eq.Wc-decomp}
  \mW_C =
  \begin{pmatrix}
    \mU_H & \mU_G & \mU_C 
  \end{pmatrix}
  \begin{pmatrix}
    \vzero & & \\
    & \vzero & \\ 
    & & \sigma_C^2 \mI 
  \end{pmatrix}
  \begin{pmatrix}
    \mU_H & \mU_G & \mU_C 
  \end{pmatrix}^\top.
 \end{equation}
 Moreover, $\mL_\rmu$ can be decomposed by $\mU_1$ as follows
 \begin{equation}
  \mL_\rmu = 
  \begin{pmatrix}
    \mU_H & \mU_G & \mU_C 
  \end{pmatrix}
  \begin{pmatrix}
    \vzero & & \\
    & \vzero & \\ 
    & & \mLambda_C 
  \end{pmatrix}
  \begin{pmatrix}
    \mU_H & \mU_G & \mU_C 
  \end{pmatrix}^\top,
 \end{equation}
 which follows that 
 \begin{equation} \label{app-eq.matern-diff-operator-decomp}
  \Big(\frac{2\nu_C}{\kappa_C^2} \mI + \mL_{\rmu} \Big)^{-\frac{\nu_C}{2}} = 
  \begin{pmatrix}
    \mU_H & \mU_G & \mU_C 
  \end{pmatrix}
  \begin{pmatrix}
    \Big(\frac{2\nu_C}{\kappa_C^2} \mI \Big)^{-\frac{\nu_C}{2}} & & \\
    & \Big(\frac{2\nu_C}{\kappa_C^2} \mI \Big)^{-\frac{\nu_C}{2}} & \\
    & & \Big(\frac{2\nu_C}{\kappa_C^2} \mI + \mLambda_C \Big)^{-\frac{\nu_C}{2}}
  \end{pmatrix}
  \begin{pmatrix}
    \mU_H & \mU_G & \mU_C 
  \end{pmatrix}^\top .
 \end{equation}
 By plugging \cref{app-eq.Wc-decomp} and \cref{app-eq.matern-diff-operator-decomp} into \cref{app-eq:cov-1}, we can then express the covariance as 
 \begin{equation}
  \begin{aligned}
    \Cov[\vf_C] & = 
    \begin{pmatrix}
      \mU_H & \mU_G & \mU_C 
    \end{pmatrix}
    \begin{pmatrix}
      \vzero & & \\
      & \vzero & \\
      & & \sigma_C^2 \Big(\frac{2\nu_C}{\kappa_C^2} \mI + \mLambda_C \Big)^{-\nu_C}
    \end{pmatrix}
    \begin{pmatrix}
      \mU_H & \mU_G & \mU_C 
    \end{pmatrix}^\top  \\
    & =  \mU_C \sigma_C^2 \Big(\frac{2\nu_C}{\kappa_C^2} \mI + \mLambda_C \Big)^{-\nu_C} \mU_C^\top 
  \end{aligned}
 \end{equation}
 which returns the \Matern~curl GP $\vf_C \sim \gG\gP \Big(\vzero, \sigma_C^2 \mU_C  \Big(\frac{2\nu_C}{\kappa_C^2} \mI + \mL_{\rmu} \Big)^{-\nu_C} \mU_C^\top \Big)$. 

 Second, consider the following SPDE 
 \begin{equation}
 e^{\frac{\kappa_C^2}{4}\mL_{\rmu}} \vf_C = \vw_C.
 \end{equation}
 Following the same procedure as above, we have its solution as
 \begin{equation}
  \vf_C \sim \gG\gP \Big(\vzero, \sigma_C^2 \mU_C  e^{-\frac{\kappa^2_C}{2}} \mU_C^\top \Big)
 \end{equation}
 which is the diffusion curl GP. 
\end{proof}

\begin{proposition}
  Given a scaled gradient white noise  $\vw_G\sim\gN(\vzero,\mW_G)$ where $\mW_G = \sigma_G^2\mU_G\mU_G^\top$, consider the following SPDE on edges: 
  \begin{equation} 
    \Phi_G(\mL_{\rmd}) \vf_G = \vw_G, 
  \end{equation}
  with differential operators
  \begin{equation}
      \Phi_G(\mL_{\rmd}) = \Big(\frac{2\nu_G}{\kappa_G^2} \mI + \mL_{\rmd} \Big)^{\frac{\nu_G}{2}}, \quad \Phi_G(\mL_{\rmd})=e^{\frac{\kappa_G^2}{4}\mL_{\rmd}}.
  \end{equation}
  The respective solutions give the curl edge GPs with \Matern~kernel and diffusion kernel
  \begin{equation} 
    \vf_G \sim \gG\gP \Big(\vzero, \sigma_G^2 \mU_G  \Big(\frac{2\nu_G}{\kappa_G^2} \mI + \mL_{\rmd} \Big)^{-\nu_G} \mU_G^\top \Big) \quad 
    \vf_G \sim \gG\gP \Big(\vzero, \sigma_G^2 \mU_G  e^{-\frac{\kappa^2_G}{2}} \mU_G^\top \Big).
  \end{equation}
\end{proposition}
\begin{proof}
  The proof follows \cref{app-prop:curl-gp-spde} likewise. 
\end{proof}

\subsection{Proof of Properties of HC Edge GPs} \label{app:proof-properties-hc-gp}
Here we provide proofs for \cref{lemma:property-hc-gp}, which directly follow from \cref{def:hodge-compositional-gp}.
\begin{proof}
    For an edge GP $\vf_1$ with covariance kernel $\mK_1$, due to the fact that the HC edge kernel $\mK_1$ is built using all the orthonormal basis of the edge function space $\mU_1$, its realizations give all possible edge functions. This is analogous to Karhunen-Loève theorem for GPs with Mercer kernels.
    For the second point that $\mK_1 = \mK_H + \mK_G + \mK_C$ and the three Hodge GPs mutually independent, this results from the construction of $\vf_1$ and the orthogonality of the three Hodge GPs. 
\end{proof}

\subsection{Posterior Distributions of Hodge Components} \label{app-subsec:posterior-hodge-components}
Here we discuss the posterior distribution of the three Hodge components from the posterior prediction of the edge function. 
As the construction of our HC edge GPs is essentially a sum of three independent functions, we can follow \citet[Section 2.4]{duvenaud2014automatic} modeling the sums of Euclidean functions. 
Denote $\vf_1(\vx)$ and $\vf_1(\vx^*)$ the function values, respectively, at training locations $\vx=[x_1,\dots,x_n]^\top$ and query locations $\vx^* = [x_1^*,\dots,x_n^*]^\top$. 
We first write down the joint prior distribution over the three Hodge components and the edge function. 
\begin{equation}
    \begin{bmatrix}
        \vf_H(\vx) \\ \vf_H(\vx^*) \\ \vf_G(\vx) \\ \vf_G(\vx^*) \\ 
        \vf_C(\vx) \\ \vf_C(\vx^*) \\ \vf_1(\vx) \\ \vf_1(\vx^*)
    \end{bmatrix}
    \sim \gN 
    \begin{pmatrix}
        % \begin{bmatrix}
        %     \vzero \\ \vzero \\ \vzero \\ \vzero \\ \vzero \\ \vzero \\ \vzero \\ \vzero \\ 
        % \end{bmatrix}, &
        \vzero, &
        \begin{bmatrix}
            \mK_H & \mK_H^* &  &  &  &  & \mK_H & \mK_H^* \\
            \mK_H^{*\top} & \mK_H^{**} & & & & & \mK_H^* & \mK_H^{**} \\ 
            & & \mK_G & \mK_G^* & & & \mK_G & \mK_G^* \\ 
            & & \mK_G^{*\top} & \mK_G^{**}  & & & \mK_G^* & \mK_G^{**} \\ 
            & & & & \mK_C & \mK_C^* & \mK_C & \mK_C^* \\
            & & & & \mK_C^{*\top} & \mK_C^{**} & \mK_C^* & \mK_C^{**} \\ 
            \mK_H & \mK_H^{*\top} & \mK_G & \mK_G^{*\top} & \mK_C & \mK_C^{*\top} & \mK_1 & \mK_1^* \\ 
            \mK_H^{*\top} & \mK_H^{**} & \mK_G^{*\top} & \mK_G^{**} & \mK_C^{*\top} & \mK_C^{**} & \mK_1^{*\top} & \mK_1^{**}
        \end{bmatrix}
    \end{pmatrix}
\end{equation}
where we represent the kernel matrices by $\mK_1 = k_1(\vx,\vx), \mK_1^*=k_1(\vx,\vx^*)$ and $\mK_1^{**}=k_1(\vx^*,\vx^*)$, and likewise for the other kernel matrices.
Given this joint distribution, we can obtain the posterior distributions of the three Hodge components as follows 
\begin{subequations}
\begin{align}
    \vf_H(\vx^*) | \vf_1(\vx) & \sim \gN \Big(\mK_H^{*\top}\mK_1^{-1}\vf_1(\vx), \mK_H^{**} - \mK_H^{*\top}\mK_1^{-1}\mK_H^* \Big) \\
    \vf_G(\vx^*) | \vf_1(\vx) & \sim \gN \Big(\mK_G^{*\top}\mK_1^{-1}\vf_1(\vx), \mK_G^{**} - \mK_G^{*\top}\mK_1^{-1}\mK_G^* \Big) \\
    \vf_C(\vx^*) | \vf_1(\vx) & \sim \gN \Big(\mK_C^{*\top}\mK_1^{-1}\vf_1(\vx), \mK_C^{**} - \mK_C^{*\top}\mK_1^{-1}\mK_C^* \Big) 
\end{align}
\end{subequations}
From these posterior distributions, we can directly obtain the means and the uncertainties of the Hodge components of the predicted edge function.

\subsection{Edge Fourier Feature Perspective} \label{app:fourier-feature-perspective}
\paragraph{Edge Fourier transform}
From the edge eigen-feature perspective, any edge function can be viewed as a linear combination of eigenvectors in $\mU$, that is, 
\begin{equation}
  \vf_1=\sum_{i=1}^{N_1}\tilde{f}_{1,i} \vu_i=\mU_1\tilde{\vf}_1 \text{ with } \tilde{\vf}_1 = \mU^\top \vf_1
\end{equation}
where $\tilde{\vf}_1$ is known as the (edge) Fourier feature of $\vf_1$ and $\tilde{f}_{1,i}$ is the $i$-th Fourier coefficient at eigenvalue $\lambda_i$. 
These eigenvalues carry the notion of frequency \citep{barbarossaTopologicalSignalProcessing2020}. 
Particularly, based on the reorganized eigenvector matrix $\mU_1 = [\mU_H\,\,\mU_G\,\,\mU_C]$ and the associated eigenvalues $\mLambda_1=\diag(\mLambda_H,\mLambda_G,\mLambda_C)$, we have that any $\lambda_G$ measures the squared $\ell_2$-norm of the divergence while $\lambda_C$ measures the squared $\ell_2$-norm of the curl:
% \begin{equation}
%   \lambda_G = \vu_G^\top\mL_1\vu_G = \vu_G^\top\mL_\rmd\vu_G = \lVert\mB_1\vu_G\rVert_2^2, \quad \lambda_C = \vu_C^\top\mL_1\vu_C = \vu_C^\top\mL_\rmu\vu_C = \lVert\mB_2^\top\vu_C\rVert_2^2,
% \end{equation} 
$\lambda_G = \vu_G^\top\mL_1\vu_G = \vu_G^\top\mL_\rmd\vu_G = \lVert\mB_1\vu_G\rVert_2^2$, and  
$\lambda_C = \vu_C^\top\mL_1\vu_C = \vu_C^\top\mL_\rmu\vu_C = \lVert\mB_2^\top\vu_C\rVert_2^2,$
and a zero eigenvalue $\lambda_H=0$ corresponding to harmonic eigenvector $\mu_H$ has zero total divergence and curl, as discussed by \citep{yangFiniteImpulseResponse2021,yangSimplicialConvolutionalFilters2022}. 
Therefore,  the Fourier coefficients at eigenvalues in different Hodge subspaces measure the weights of the corresponding Fourier basis in $\vf$, each basis associated with different total divergence or total curl. That is, we have the edge Fourier representation as  
\begin{equation}
  \tilde{\vf}_1 = \mU_1^\top \vf_1 = [\tilde{\vf}_H^\top,\tilde{\vf}_G^\top, \tilde{\vf}_C^\top]^\top \text{ with } \tilde{\vf}_H = \mU_H^\top\vf_1, \,\, \tilde{\vf}_G = \mU_G^\top\vf_1,  \,\, \tilde{\vf}_C = \mU_C^\top\vf_1.
\end{equation}
\begin{remark}
  This provides as a spectral tool to understand the edge functions. 
  The harmonic Fourier feature $\tilde{\vf}_H$ measures the extent of harmonic Fourier basis $\mU_H$ in $\vf$, reflecting how harmonic $\vf_1$ is.
  The gradient Fourier feature $\tilde{\vf}_G$ measures the extent of gradient Fourier basis $\mU_G$ in $\vf_1$, reflecting how divergent $\vf_1$ is, where each basis in $\mU_G$ has different total divergence. 
  The curl Fourier feature $\tilde{\vf}_C$ measures the extent of curl Fourier basis $\mU_C$ in $\vf_1$, reflecting how rotational $\vf_1$ is, where each basis in $\mU_C$ has different total curl. 
\end{remark}

\begin{corollary}[Fourier feature perspective of edge GPs]
  Let $\vf_1\sim \gG\gP(\vzero,\mK_1)$ be an edge Gaussian process with kernel diagonalizable by $\mU_1$. Then, given the edge Fourier transform $\tilde{\vf}_1=\mU_1^\top\vf_1$, its Fourier coefficients $\{\tilde{f}_{1,i}\}_{i=1}^N$ are independently distributed Gaussian variables
  \begin{equation}
    \tilde{f}_{1,i} \sim \gN(0, \vu_i^\top \mK_1 \vu_i), \text{ for }i=1,\dots,N. 
  \end{equation}
\end{corollary}
\begin{proof}
  Using the affine transformation preserving Gaussian, we have 
  \begin{equation}
    \tilde{\vf}_1 \sim \gG\gP(\vzero, \mU_1^\top\mK_1\mU_1).
  \end{equation}
  Since the kernel $\mK_1$ can be diagonalized by $\mU_1$, the kernel $\mU_1^\top\mK_1\mU_1$ is a diagonal matrix, implying the independence between variables in $\tilde{\vf}_1$. Thus, a variable $ \tilde{f}_{1,i}$ follows normal distribution $\gN(0, \vu_i^\top \mK_1 \vu_i)$.  
\end{proof}
This corollary indicates that an edge GP can be viewed as an affine transformation by $\mU_1$ of a collection of independent Gaussian variables, $\tilde{\vf}_1 = [ \tilde{f}_{1,1},\dots, \tilde{f}_{1,N_1}]^\top$, which are the Fourier coefficients of $\vf$.
The prior distribution of certain Fourier coefficient is the  prior imposed on the corresponding divergent or rotational part of the function $\vf$.
This allows us to compare HC and non-HC edge GPs from the following perspective. 

\begin{proposition} 
  Suppose the Hodge Laplacian $\mL_1$ has eigenpairs $(\lambda,\vu_G)$ and $(\lambda,\vu_C)$, i.e., $\lambda$ is associated to both gradient and curl subspaces. Let $\vf_1\sim \gG\gP(\vzero,\mK_1)$ be an edge Gaussian process. Denote the Fourier coefficients of $\vf_1$ at $\lambda_G$ and $\lambda_C$ as $\tilde{f}_G$ and $\tilde{f}_C$, respectively.
  Then, a non-Hodge-compositional GP with $\mK_1=\Psi(\mL_1)$ imposes the same prior variance on $\tilde{f}_G$ and $\tilde{f}_C$, i.e.,
  \begin{equation}
    \Var[\tilde{f}_G] = \Var[\tilde{f}_C] = \Psi(\lambda).
  \end{equation}
  Instead, a Hodge-compositional GP with $\mK_1$ in \cref{eq.hodge-gp-per-comp} imposes different variances on two coefficients 
  \begin{equation}
    \Var[\tilde{f}_G] = \Psi_G(\lambda) \text{ and } \Var[\tilde{f}_C] = \Psi_C(\lambda).
  \end{equation}
\end{proposition}
\begin{proof}
  For a non-HC edge GP with kernel $\Psi(\mL_1)$, its Fourier coefficients $\tilde{f}_G$ and $\tilde{f}_C$ at a common $\lambda$ follows the normal distribution with a variance $\Psi(\lambda)$, which follows from the nature of kernel function $\Psi$ mapping each $\lambda$ to exactly one value $\Psi(\lambda)$. However, for a HC edge GP, we have 
  \begin{equation}
    \tilde{f}_G\sim\gN(0,\vu_G^\top\mK_1\vu_G), \quad \tilde{f}_C\sim\gN(0,\vu_C^\top\mK_1\vu_C). 
  \end{equation}
  Using \cref{def:hodge-compositional-gp} [cf. \cref{eq.hodge-gp-per-comp}], we have $\vu_G^\top\mK_1\vu_G = \Psi_G(\lambda)$ and $\vu_C^\top\mK_1\vu_C = \Psi_C(\lambda)$, which are two different values, arising from the individually parametrized kernels $\mK_G$ and $\mK_C$. 
\end{proof}
This edge Fourier feature perspective directly shows that non-HC GPs impose the same prior on two Fourier coefficients, which are however associated with two different Hodge subspaces. 
This prohibits individual learning for the gradient and curl parts of edge functions particularly associated to the same eigenvalue. 
Instead, HC edge GPs do not have this limitation.

\subsection{Diffusion on Edges} \label{app:diffusion_on_edges}
Here we provide the details on the connection of diffusion HC edge GPs to edge diffusion equations, as well as an illustration of diffusion process on edges.
Consider the diffusion equation on the edge space 
\begin{equation}  
  \odv{\vphi(t)}{t} =  -(\mu \mL_{\rmd} + \gamma \mL_{\rmu}) \vphi(t)
\end{equation}
where $\mu,\gamma>0$.
Given an initial value $\vphi(0)$, we obtain a solution 
\begin{equation}
  \vphi|_{t=\tau} = e^{-(\mu\tau\mL_\rmd + \gamma\tau\mL_\rmu)}\vphi(0),
\end{equation}
When $\sigma_G^2 = \sigma_C^2 = \sigma_H^2=1$, the diffusion kernel can be written as 
\begin{equation}
  \mK_1 = e^{-(\frac{\kappa_G^2}{2} \mL_{\rmd} + \frac{\kappa_C^2}{2} \mL_{\rmu})}
\end{equation}
which is the Green's function of above diffusion equation.
In \cref{fig:diffusion_demonstration}, we illustrate the diffusion processes on nodes and on edges, started at a random location. 
When the graph is connected, the node diffusion converges to the harmonic state where all nodes are constant. 
Instead, the harmonic state of the edge diffusion gives an edge flow which is div- and curl-free, cycling around the 1-dimensional ``hole'' of the $\textSC$ \citep{munkresElementsAlgebraicTopology2018}.

% This equation can be Hodge-decomposed into three parts 
% \begin{equation*}
%     \odv{\vf_{G}(t)}{t} = -\kappa_G \mL_{\rmd} \vf_{G}(t), \quad  \odv{\vf_{C}(t)}{t} = -\kappa_C \mL_{\rmu} \vf_{C}(t), \quad \odv{\vf_{H}(t)}{t} = \vzero,
% \end{equation*}
% which can be solved in terms of Hodge components as follows 
% \begin{equation}
%   \vf_{G}(t) = e^{-\kappa_G\mL_{\rmd}t}\vf_{G}(0) \,\, \vf_{C}(t) = e^{-\kappa_C\mL_{\rmu}t}\vf_{C}(0)  \,\, \vf_H(t) = \vf_{H}(0),
% \end{equation}

% When $\sigma^2_\Box$ are not necessarily ones, this corresponds to the case that initial condition follows that 

\begin{figure}[ht!]
\makebox[1\textwidth][l]{
  \hspace{-10pt}
  \includegraphics[width=0.245\linewidth]{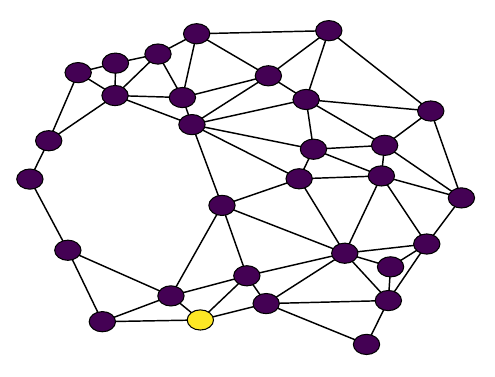}
    \hspace{-10pt}
  \includegraphics[width=0.245\linewidth]{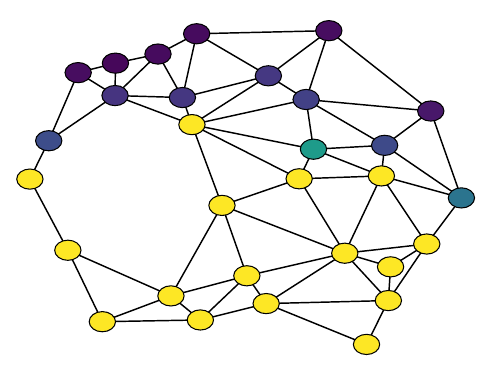}
    \hspace{-10pt}
  \includegraphics[width=0.245\linewidth]{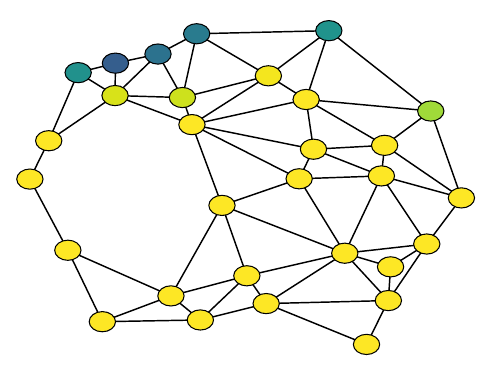}
    \hspace{-10pt}
  \includegraphics[width=0.245\linewidth]{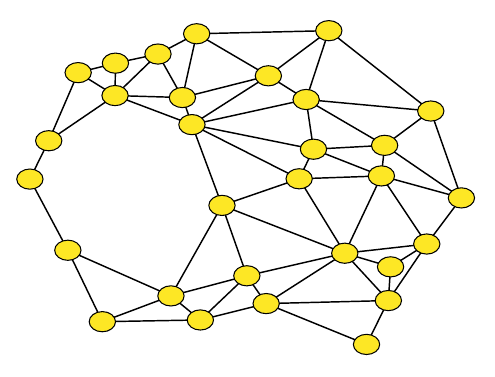} 
    \hspace{-10pt}
  \includegraphics[width=0.046\linewidth]{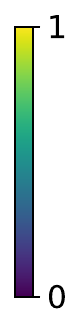} 
  }
  % \makebox[1\textwidth][l]{
  %   \hspace{-10pt}
  %   \includegraphics[width=0.245\linewidth]{figures/diffusion_illustration/linegraph/heat_0.pdf}
  %     \hspace{-10pt}
  %   \includegraphics[width=0.245\linewidth]{figures/diffusion_illustration/linegraph/heat_1.pdf}
  %     \hspace{-10pt}
  %   \includegraphics[width=0.245\linewidth]{figures/diffusion_illustration/linegraph/heat_2.pdf}
  %     \hspace{-10pt}
  %   \includegraphics[width=0.245\linewidth]{figures/diffusion_illustration/linegraph/heat_4.pdf} 
  %     \hspace{-10pt}
  %   \includegraphics[width=0.046\linewidth]{figures/diffusion_illustration/linegraph/colorbar.pdf} 
  %   }
\makebox[1\textwidth][l]{
\hspace{-10pt}
  \includegraphics[width=0.245\linewidth]{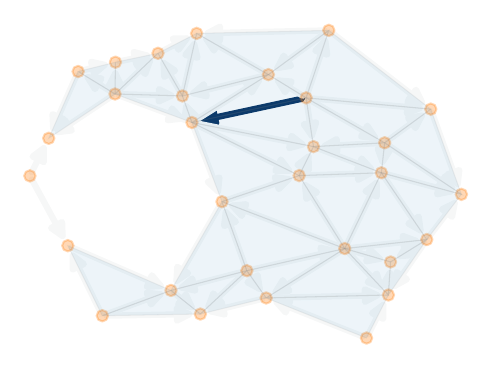}
  \hspace{-10pt}
  \includegraphics[width=0.245\linewidth]{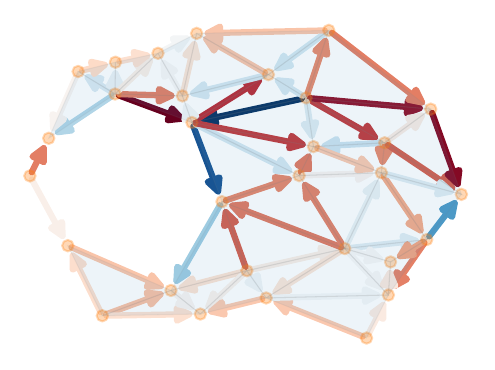} 
  \hspace{-10pt}
  \includegraphics[width=0.245\linewidth]{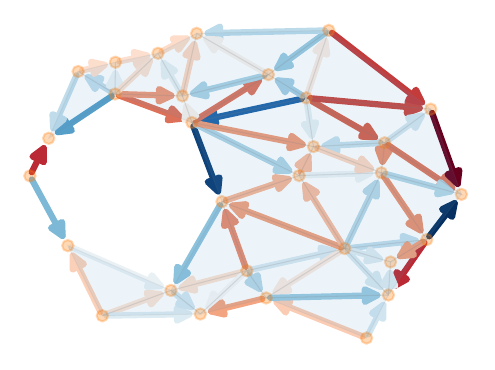} 
  \hspace{-10pt}
  \includegraphics[width=0.245\linewidth]{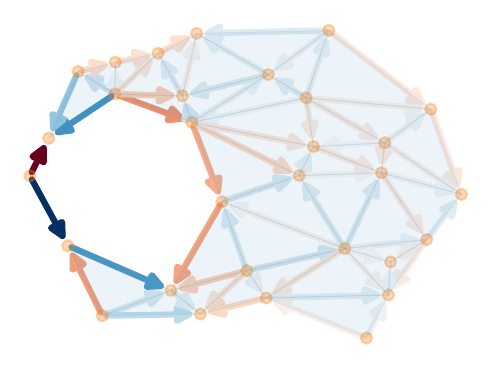}
  \hspace{-10pt}
\includegraphics[width=0.055\linewidth]{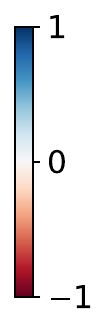} 
  }
  \caption{Node \figtop~and edge \figbottom~diffusion processes (started at one random location \figleft, then two intermediate states \figcenter~and harmonic state \figright).}
  \label{fig:diffusion_demonstration}
\end{figure}

\subsection{Complexity of Edge GPs} \label{app-subsec:complexity}
% \paragraph{Efficiency}
Here we discuss their complexity when training, e.g., in Gaussian process regression, and the complexity of sampling from them.
Note that the complexity of graph GPs naturally apply to edge GPs.

\paragraph{Complexity when Training}
The \Matern~and diffusion kernels can be trained in a scalable way. 
Due to their decreasing eigenvalues, we can consider the $l$ largest eigenvalues of the kernel matrices with off-the-shelf eigen-solvers, e.g., Lanczos algorithm. 
The recent work on Krylov subspace methods to accelerate graph kernels by \citet{erbKrylovSubspaceMethods2023} can be extended to edge kernels. 
Moreover, other computational techniques applicable for graph GPs in \citet[Section 3.1]{borovitskiyMatErnGaussian2021} can be adopted as well. 

\paragraph{Complexity when Sampling from Edge GPs}
% A standard method for sampling from (edge) GPs involves a Cholesky decomposition of the kernel. The computational cost of sampling is a sum of $\gO(n^3)$ on Cholesky factorization and $\gO(n^2)$ (matrix-vector multiplication for each sample, where $n$ is the number of points. 
Given an edge GP, as well as the eigenpairs for constructing the edge kernel, we can follow the procedure in \cref{proof:prop-gp-samples} to sample an edge function. That is, from a vector $\vv=(v_1,\dots,v_{N_1})^\top$ of variables following independent normal distribution, a sample of the edge function can be given by 
\begin{equation}
      \vf_1 = [\mU_H \,\, \mU_G \,\, \mU_C] \, \diag([\Psi_H^\frac{1}{2}(\mLambda_H),\Psi_G^\frac{1}{2}(\mLambda_G),\Psi_C^\frac{1}{2}(\mLambda_C)]) \vv 
\end{equation}
which has a complexity of $\gO(N_1^2)$ (matrix-vector multiplication).
Furthermore, the discussion on improving sampling efficiency in graph GP models by \citet[Section 4.7]{nikitinNonseparableSpatiotemporalGraph2022} naturally applies to our proposed edge GPs as well.

\subsection{Interaction between Node, Edge and Triangle GPs} \label{app:interaction-node-edge}
Here we provide the proof for \cref{cor:gradient-of-node-gp}, showing the gradient of a node GP is an edge GP. 
\begin{proof}
  Given a node GP $\vf_0\sim\gG\gP(\vzero,\mK_0)$, using the derivative of a GP is also a GP, its gradient $\vf_G=\mB_1^\top\vf_0$ is an edge GP whose kernel can be found as 
  \begin{equation}
    \mK_G = \Cov[\vf_G] = \mB_1^\top \Cov[\vf_0] \mB_1 = \mB_1^\top\mK_0\mB_1.
  \end{equation}
  % Consider the singular value decomposition $\mB_1=\mU_0\mSigma_0\mV_0^\top$. We have 
  % \begin{equation}
  %   \mL_0 = \mB_1\mB_1^\top = \mU_0 \mSigma_0\mSigma_0^\top \mU_0^\top \text{ and } \mK_0 = \Psi(\mL_0) = \mU_0\Psi(\mLambda_0)\mU_0^\top \text{ with } \mLambda_0 = \mSigma_0 \mSigma_0^\top .
  % \end{equation}
  % Therefore, we can write 
  % \begin{equation}
  %   \mB_1^\top\mK_0\mB_1 = \mV_0\mSigma_0^\top \mU_0^\top \mU_0 \Psi(\mSigma_0 \mSigma_0^\top ) \mU_0^\top \mU_0 \mSigma_0\mV_0^\top = \mV_0\mSigma_0^\top\Psi(\mSigma_0 \mSigma_0^\top )\mSigma_0\mV_0^\top  = 
  % \end{equation}
  By definition, $\mL_0 = \mB_1\mB_1^\top$ and $\mL_\rmd = \mB_1^\top \mB_1$ are isospectral, having the same nonzero eigenvalues. 
  Furthermore, using $\mK_0 = \Psi_0(\mL_0)$, we can write above covariance as 
  \begin{equation}
    \mK_G = \mB_1^\top \Psi(\mB_1\mB_1^\top) \mB_1 = \mB_1^\top\mB_1 \Psi_0(\mB_1^\top\mB_1) = \mL_\rmd \Psi_0(\mL_\rmd)
  \end{equation}
  where the second equality can be shown by using the definition of analytic functions of matrix \citep[Corollary 1.34]{higham2008functions}. 
  Furthermore, relying on the eigendecomposition 
  \begin{equation}
    \mL_\rmd = 
    \begin{pmatrix}
      \mU_H & \mU_G & \mU_C 
    \end{pmatrix}
    \begin{pmatrix}
      \vzero & & \\
      & \mLambda_G & \\ 
      & & \vzero
    \end{pmatrix}
    \begin{pmatrix}
      \mU_H & \mU_G & \mU_C 
    \end{pmatrix}^\top,
  \end{equation}
  we can obtain 
  \begin{equation}
    \mK_G = \mU_G \mLambda_G \Psi_0(\mLambda)\mU_G^\top ,
  \end{equation}
  which gives the gradient kernel function $\Psi_G(\mLambda_G) = \mLambda_G\Psi_0(\mLambda_G)$. 
\end{proof}

In the following we provide the respective corollaries for other derivative operations of interest, where the proofs can directly follow from the fact that derivatives preserve Gaussianity. 
\begin{corollary}[Curl of a triangle GP] \label{app-cor:curl-adjoint-triangle-gp}
  Suppose a triangle function $\vf_2$ is a GP $\vf_2\sim\gG\gP(\vzero,\mK_2)$ with $\mK_2=\Psi_2(\mL_2)=\mU_2 \Psi_2(\mLambda_2)\mU_2^\top$ given the eigendecomposition $\mL_2=\mU_2\mLambda_2\mU_2^\top$. 
  Then, its curl is an edge GP $\vf_C\sim\gG\gP(\vzero,\mK_C)$ where $\mK_C = \mU_C\Psi_C(\mLambda_C)\mU_C^\top$ with 
  \begin{equation}
    \Psi_C(\mLambda_C) = \mLambda_C \Psi_2(\mLambda_C).
  \end{equation}
\end{corollary}

\begin{proposition}[Div and Curl of edge GPs]\label{app-prop:div-curl-gp}
  Let $\vf_1\sim \gG\gP(\vzero,\mK)$ be a Hodge-compositional edge Gaussian process in \cref{def:hodge-compositional-gp}. Then, its divergence and curl are Gaussian processes on nodes and triangles, respectively, as follows 
  \begin{equation}
      \mB_1 \vf \sim \gG\gP(\vzero, \mB_1\mK_G\mB_1^\top), \quad 
      \mB_2^\top \vf \sim \gG\gP (\vzero, \mB_2^\top \mK_C\mB_2).
  \end{equation}
\end{proposition}

% \begin{proposition}
%   Let $\vf\sim \gG\gP(\vm,\mK)$ be an edge Gaussian process. Then, its divergence and curl are also Gaussian processes with
%   \begin{subequations}
%     \begin{align}
%       \textDiv \vf & \sim \gG\gP(\mB_1\vm_G, \mB_1\mK_G\mB_1^\top) \text{ and}\\ 
%       \textCurl \vf & \sim \gG\gP (\mB_2^\top\vm_C, \mB_2^\top \mK_C\mB_2),
%     \end{align}
%   \end{subequations}
%   where $\vm_G$ and $\vm_C$ are the gradient and curl components of $\vm$, respectively.
% \end{proposition}
 
\begin{remark}
  These interactions between GPs on nodes, edges and triangles provide us alternative ways to construct gradient and curl edge GPs [cf. \cref{cor:gradient-of-node-gp,app-cor:curl-adjoint-triangle-gp}], as well as construct appropriate node GPs and triangle GPs. 
  They are more applicable when the underlying physical relationships exist between the corresponding functions and the GP priors on the original simplices are easier to construct.
\end{remark}

\subsection{Alternative Hodge-compositional Edge GPs} \label{app:alternative_gp}
Here we provide the proof for \cref{prop:alternative-construction} giving an alternative way to build HC edge GPs. 
% \begin{proposition} \label{prop:alternative-construction-2}
%   Let $\vf_1$ be an edge function defined in \cref{eq.hodge-composition-expression} with harmonic component $\vf_H$, node function $\vf_0$ and triangle function $\vf_2$. 
%   If we model $\vf_0$ as a GP on nodes $\vf_0 \sim \gG\gP(\vzero, \mK_0)$, model $\vf_2$ as a GP on triangles $\vf_2 \sim \gG\gP(\vzero, \mK_2)$, and $\vf_H$ as a harmonic GP $\vf_H\sim \gG\gP(\vzero,\mK_H)$, then we have GP $\vf_1 \sim \gG\gP(\vzero, \mK_1)$ with  
%   \begin{equation}
%     \mK_1 = \mK_H + \mB_1^\top\mK_0\mB_1 + \mB_2 \mK_2 \mB_2^\top.
%   \end{equation}
% \end{proposition}
\begin{proof}
  From the Hodge decomposition, we can write an edge function as 
  \begin{equation}
    \vf_1 = \vf_H + \mB_1^\top\vf_0 + \mB_2\vf_2.
  \end{equation}
  where $\vf_0$ and $\vf_2$ are some node and triangle functions.
  Then, the proof can be completed by using the results from \cref{cor:gradient-of-node-gp,app-cor:curl-adjoint-triangle-gp}. 
\end{proof}

\subsection{Alternative HC Edge GPs from SPDEs on Edges}
While gradient and curl edge GPs in \cref{def:hodge-compositional-gp} can be linked to their SPDEs as discussed by \cref{prop:spde-grad-curl-gp}, we can also obtain the alternatively constructed counterparts in \cref{cor:gradient-of-node-gp,app-cor:curl-adjoint-triangle-gp} from SPDEs. 
Again, we consider the \Matern~family.
\begin{corollary}\label{app-cor:spde-node-noise}
  Suppose a node function $\vf_0$ is a graph (node) \Matern~GP $\vf_0\sim\gG\gP(\vzero,\mK_0)$ with 
  \begin{equation}
    \mK_0=\Psi_0(\mL_0)=\Big(\frac{2\nu_0}{\kappa^2_0}\mI + \mL_0 \Big)^{-\nu_0}.
  \end{equation}
  Then, \cref{cor:gradient-of-node-gp} gives us its gradient as a gradient edge GP $\vf_G\sim\gG\gP(\vzero,\mK_G)$  with 
  \begin{equation}
    \mK_G = \mL_\rmd \Big(\frac{2\nu_0}{\kappa^2_0}\mI + \mL_\rmd \Big)^{-\nu_0}.
  \end{equation}
  Furthermore, the gradient GP $\vf_G$ is the solution of the following SPDE 
  \begin{equation}
    \Big(\frac{2\nu_0}{\kappa^2_0}\mI + \mL_\rmd \Big)^{\frac{\nu_0}{2}} \vf_G = \mB_1^\top \vw_0 
  \end{equation}
  where $\vw_0$ is a standard Gaussian noise on nodes following $\vf_0\sim\gN(\vzero,\mI)$.
\end{corollary}
\begin{proof}
  First, we can solve the SPDE with the following solution 
  \begin{equation}
    \vf_G =  \Big(\frac{2\nu_0}{\kappa^2_0}\mI + \mL_\rmd \Big)^{-\frac{\nu_0}{2}} \mB_1^\top \vw_0 
    = 
    \mB_1^\top \Big(\frac{2\nu_0}{\kappa^2_0}\mI + \mL_0 \Big)^{-\frac{\nu_0}{2}} \vw_0 
  \end{equation}
  where the second equality follows from the definition of $\mL_0$ and $\mL_\rmd$. 
  Given that $\vw_0$ is a GP, so is $\vf_G$ and we can study its covariance as 
  \begin{equation}
    \begin{aligned}
      \Cov[\vf_G] & = \mB_1^\top \Big(\frac{2\nu_0}{\kappa^2_0}\mI + \mL_0 \Big)^{-\frac{\nu_0}{2}} \Cov[\vw_0] \Big(\frac{2\nu_0}{\kappa^2_0}\mI + \mL_0 \Big)^{-\frac{\nu_0}{2}} \mB_1 \\
      & = \mB_1^\top \Big(\frac{2\nu_0}{\kappa^2_0}\mI + \mL_0 \Big)^{-\nu_0} \mB_1 \\
      & = \mL_\rmd \Big(\frac{2\nu_0}{\kappa^2_0}\mI + \mL_\rmd \Big)^{-\nu_0}
    \end{aligned}
  \end{equation}
  which completes the proof. 
\end{proof}

For completeness, we give the corollary relating the curl \Matern~edge GP obtained from some triangle GP to its SPDE representation. 
\begin{corollary}
  Suppose a triangle function $\vf_2$ is a triangle \Matern~GP $\vf_2\sim\gG\gP(\vzero,\mK_2)$ with 
  \begin{equation}
    \mK_2=\Psi_2(\mL_2)=\Big(\frac{2\nu_2}{\kappa^2_2}\mI + \mL_2 \Big)^{-\nu_2}.
  \end{equation}
  Then, \cref{app-cor:curl-adjoint-triangle-gp} gives us its curl adjoint as a curl edge GP $\vf_C\sim\gG\gP(\vzero,\mK_C)$  with 
  \begin{equation}
    \mK_C = \mL_\rmu \Big(\frac{2\nu_2}{\kappa^2_2}\mI + \mL_\rmu \Big)^{-\nu_2}.
  \end{equation}
  Furthermore, the curl GP $\vf_C$ is the solution of the following SPDE 
  \begin{equation}
    \Big(\frac{2\nu_2}{\kappa^2_2}\mI + \mL_\rmu \Big)^{\frac{\nu_2}{2}} \vf_C = \mB_2 \vw_2 
  \end{equation}
  where $\vw_2$ is a standard Gaussian noise on triangles following $\vf_2\sim\gN(\vzero,\mI)$.
\end{corollary}
\begin{proof}
    The proof can follow the same procedure as above for \cref{app-cor:spde-node-noise}.
\end{proof}

% \newpage
\section{EXPERIMENTS} \label{app:experiment_details}
Here we provide additional details on the three experiments presented in the main text. 

\paragraph{Experimental Setup}
In our three experiments we consider the regression tasks and implement GP regression using the \texttt{GPyTorch} library \citep{gardner2018gpytorch}.
We optimize the marginal log likelihood loss for 1000 iterations with the \texttt{ADAM} optimizer where the learning rate is set to the default value of 0.001. 
We run each experiment 10 times with hyperparameters randomly initialized.
We report evaluation metrics averaged over 10 experiments and the respective standard deviations.  
All experiments are run on a NVIDIA GeForce RTX 3080 GPU with 10GB of memory.

\paragraph{Line-graph Construction}
Given the incidence matrix $\mB_1$ of the original graph, the adjacency matrix and the corresponding graph Laplacian of the line-graph can be found as $\mA_{lg}:= |\mB_1^\top\mB_1-2\mI|$ and $\mL_{lg} = \diag(\mA_{lg}\vone)-\mA_{lg}$.

\subsection{Additional Details for the Forex Experiment}
In the forex experiment, we obtain the data from \textit{Foreign Exchange Data} by \textit{Oanda Corporation}\footnote{\url{https://www.oanda.com/}.}. 
The data was collected at 2018/20/05 17:00 UTC by \citet{jiaGraphbasedSemiSupervisedActive2019}.
It includes the pairwise exchange rates between the 25 most traded currencies, which form 210 exchangeable pairs.
With them as nodes and edges, we then construct an unweighted $\textSC$ by including the triangles formed by any three pairwise exchangeable currencies. 
For an edge $\{i,j\}$ connecting currencies $i,j$, we assign its orientation based on an alphabetical order of their currency names, and likewise for a triangle. 
For each exchangeable pair, we consider the underlying edge flow as $f_1(i,j) = \log r^{i/j}$, translating the arbitrage-free condition to curl-free condition, where $r^{i/j}$ is the midpoint between ask and bid prices. 
\cref{app-fig:forex_different_training_ratio} shows the prediction RMSEs using different GP models with respect to training ratios from 0.1 to 0.5 with a step 0.05, as well as the learned \Matern~kernels. 

For visualizing the predictions using different models, we consider a smaller market for better visibility where we first randomly removed seven currencies then half of the exchangeable pairs, resulting 18 currencies and 77 pairs, as shown in \cref{app-fig:forex-smaller-market}.

\begin{figure}[t!]
  \begin{subfigure}{0.33\linewidth}
    \includegraphics[width=1\linewidth]{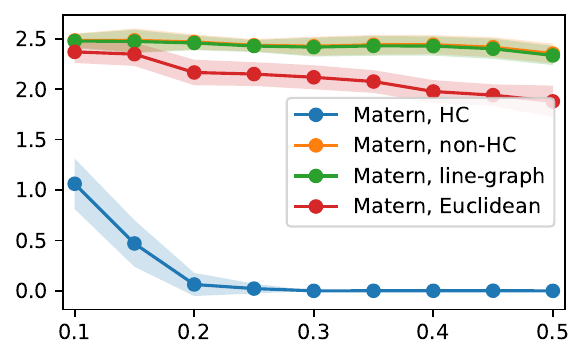} 
    \caption{RMSEs versus training ratios}
  \end{subfigure}
  \begin{subfigure}{0.33\linewidth}
    \includegraphics[width=1\linewidth]{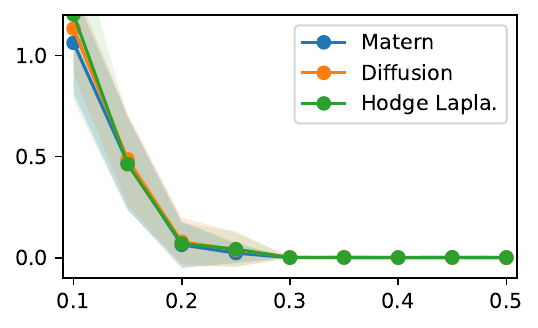} 
    \caption{RMSEs versus training ratios}
  \end{subfigure}
  % \begin{subfigure}{0.33\linewidth}
  %   \includegraphics[width=1\linewidth]{figures/forex/r2_train_ratio.pdf} 
  %   \caption{R2 scores versus train ratio}
  % \end{subfigure}
  \begin{subfigure}{0.3\linewidth}
    \includegraphics[width=1\linewidth]{figures/forex/learned_kernel_spectral.pdf} 
    \caption{Learned \Matern~kernels}
  \end{subfigure}
  \caption{(a) Forex prediction RMSEs of different GPs using \Matern~kernels with respect to training ratios. (b) Forex prediction RMSEs of HC GPs using different edge kernels with respect to training ratios. (c) Learned HC and non-HC \Matern~kernels in the spectrum for a training ratio of $0.2$.}
  \label{app-fig:forex_different_training_ratio}
\end{figure}

\begin{figure}[t!]
  \begin{subfigure}{0.32\linewidth}
    \includegraphics[width=1\linewidth]{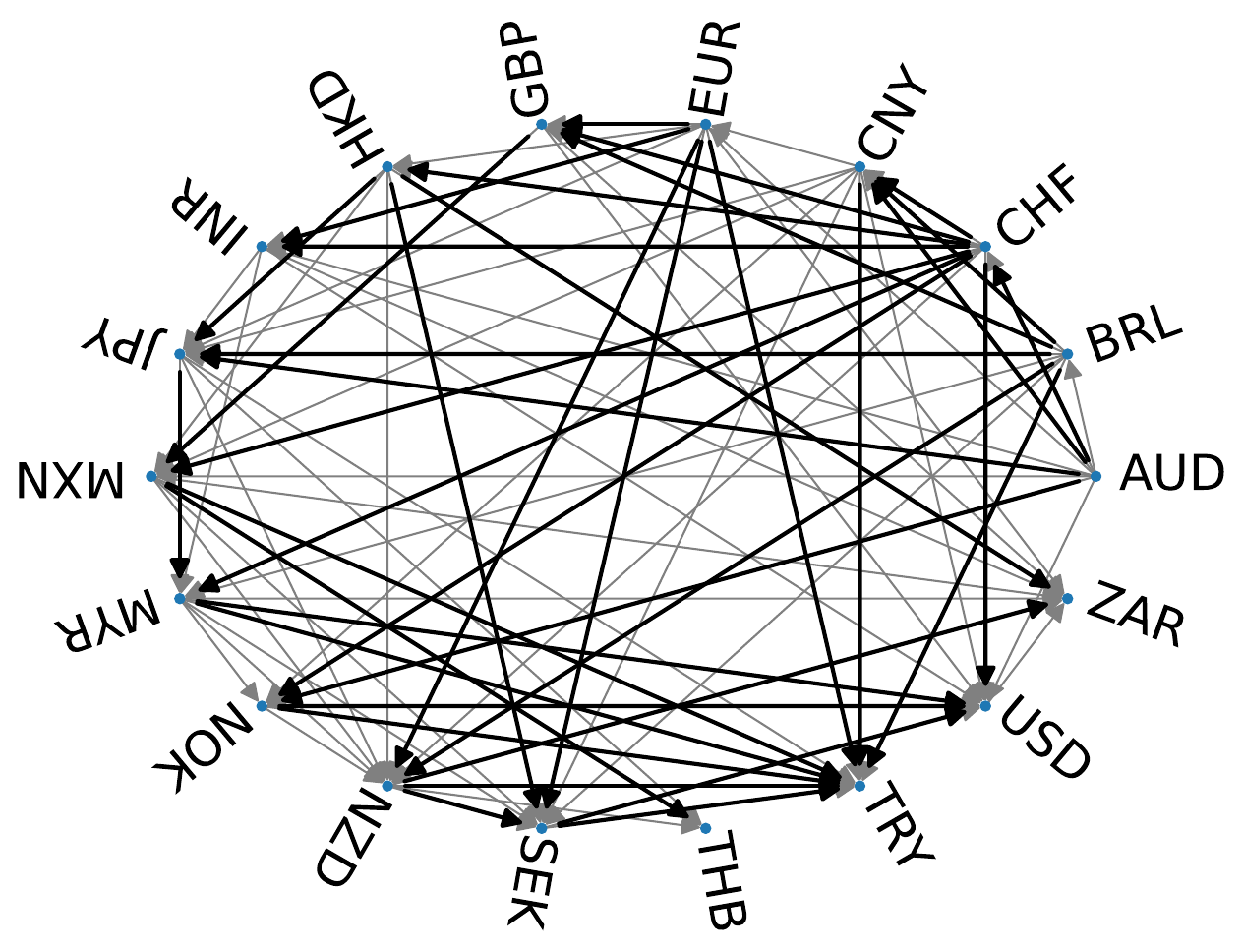} 
    \caption{Forex market $\textSC$}
  \end{subfigure}
  \begin{subfigure}{0.33\linewidth}
    \includegraphics[width=1\linewidth]{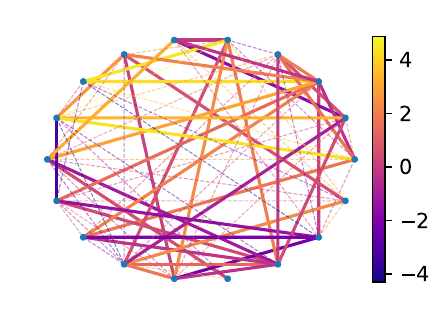} 
    \caption{Ground truth}
  \end{subfigure}
  \begin{subfigure}{0.33\linewidth}
    \includegraphics[width=1\linewidth]{figures/forex_smaller/mean.pdf} 
    \caption{HC \Matern, posterior mean}
  \end{subfigure}
  \begin{subfigure}{0.33\linewidth}
    \includegraphics[width=1\linewidth]{figures/forex_smaller/joint_matern_mean.pdf} 
    \caption{non-HC \Matern, posterior mean}
  \end{subfigure}
  % \begin{subfigure}{0.33\linewidth}
  %   \includegraphics[width=1\linewidth]{figures/forex_smaller/lg_matern_mean.pdf} 
  %   \caption{Line-graph \Matern, posterior mean}
  % \end{subfigure}
  \begin{subfigure}{0.34\linewidth}
    \includegraphics[width=1\linewidth]{figures/forex_smaller/std.pdf} 
    \caption{HC \Matern, posterior std}
  \end{subfigure}
  \begin{subfigure}{0.33\linewidth}
    \includegraphics[width=0.98\linewidth]{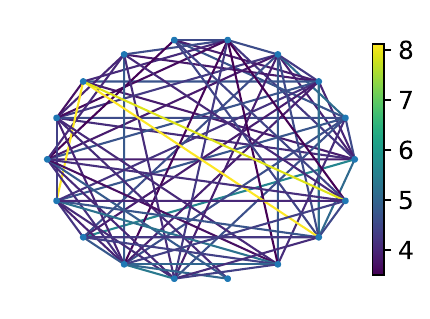} 
    \caption{Prior variance of HC \Matern~GP}
  \end{subfigure}
  % \begin{subfigure}{0.33\linewidth}
  %   \includegraphics[width=0.9\linewidth, right]{figures/forex_smaller/curl_kernel_matern.pdf} 
  %   \caption{Curl of \Matern~kernel}
  % \end{subfigure}
  % \begin{subfigure}{0.33\linewidth}
  %   \includegraphics[width=0.9\linewidth, right]{figures/forex_smaller/curl_kernel_diffusion.pdf} 
  %   \caption{Curl of diffusion kernel}
  % \end{subfigure}
  \caption{(a-e): Visualization of forex rates predictions in a smaller market. (f): Prior variance of the learned HC \Matern~GP. Note that (b-d) and (f) are the same as the ones in the main content [cf. \cref{fig:visualization_forex}]. Here we show them with a better resolution.}
  \label{app-fig:forex-smaller-market}
  %\vspace{-2mm}
\end{figure}

\subsection{Additional Details for Ocean Current Analysis}
In the second experiment, we consider the ocean drifter data, also known as \textit{Global Lagrangian Drifter Data}, which was collected by NOAA Atlantic Oceanographic and Meteorological Laboratory\footnote{\url{http://www.aoml.noaa.gov/envids/gld/}.}.
Each point in the dataset is a buoy at a specific time, with buoy ID, location (in latitude and longitude), date/time, velocity and water temperature. 
We consider the buoys that were in the North Pacific ocean dated from 2010 to 2019 with a size of around three million.
The dataset itself is a 3D point cloud after converting the location to the \textit{earth-centered, earth-fixed} (ECEF) coordinate system. 
We follow the procedure in \citet{chen2021decomposition} to first sample 1,500 buoys furthest from each other, then construct a weighted $\textSC$ as a Vietoris-Rips (VR) complex with $N_1$ around 20k and $N_2$ around 90k.
We then convert the velocity field into flows on the edges of $\textSC$ by using the linear integration approximation \citep{chenHelmholtzianEigenmapTopological2021}.
We randomly sample $20\%$ of the edges for training and test on the rest. 
To efficiently construct the edge kernels, we use eigensolver in \citet{knyazev2001toward}, implemented using the \texttt{megaman} library \citep{mcqueen2016megaman}, to compute the eigenpairs associated to the 500 largest eigenvalues. 
We evaluate the prediction mean and uncertainty in the edge flow domain, reported in \cref{app-tab:ocean-flow-full-results}.
Furthermore, we obtain the gradient and curl components of the edge flow of the prediction as in \cref{app-subsec:posterior-hodge-components}.
We visualize the predictions in the edge flow domain in \cref{app-fig:ocean-flow-full-figures-edge-flow}.
We see that both HC and non-HC edge diffusion GPs give close performance and they capture the general pattern of the edge flow. 
Moreover, the standard deviation is small in most of the locations except few locations (small islands around the lower left corner) where the edge flows (velocity fields) exhibit more discontinuities due to the boundary.

We further convert the edge flows back into vector fields, as shown in \cref{app-fig:ocean-flow-full-figures}. 
We refer to \citet{chenHelmholtzianEigenmapTopological2021} for this procedure. 
We also approximate the standard deviation of the velocity field prediction by sampling 50 edge flows from the posterior distribution and converting them to the vector field domain, followed by computing the average $\ell_2$ distance between the samples and the mean per location, as shown in \cref{app-fig:std_pointwise_approx}.

% which solves an overdetermined linear system and does not recover the absolute scales of the original vector fields, but rather the relative relations. 

% \begin{enumerate}
%     \item eigensolver: Locally Optimal Block Preconditioned Conjugate Gradient Method \citep{knyazev2001toward}
%     \item packages: megaman \citep{mcqueen2016megaman}
% \end{enumerate}

% \begin{table}[hpt!]
%   \centering
%   \caption{Ocean flow inference results.}
%   \resizebox{1\linewidth}{!}{
%     \begin{tabular}{lllllllllll}
%       \toprule
%       \multirow{2}{*}{Method}  & \multicolumn{2}{c}{MAE } & \multicolumn{2}{c}{MSE } & \multicolumn{2}{c}{R2 } & \multicolumn{2}{c}{MSLL } & \multicolumn{2}{c}{NLPD}  \\ 
%       \cmidrule(lr){2-11}  & Diffusion & \Matern~ &  Diffusion & \Matern~ &  Diffusion & \Matern~ &  Diffusion & \Matern~ &  Diffusion & \Matern~\\
%       \midrule      
%       Euclidean & $0.76\pm0.00$ & $0.76\pm0.00$ & $1.00\pm 0.01$ & $1.00\pm 0.01$ & --- & --- & $1.42\pm0.01$ & $1.42\pm0.00$ & $1.42\pm0.01$ & $1.42\pm0.10$\\
%       Non-Hodge & $0.25\pm0.00$ & $0.26\pm 0.00$ & $0.12\pm0.00$ & $0.12\pm 0.00$ & $0.86\pm 0.00$ & $0.87\pm 0.00$& $0.41\pm 0.01$ & $1.36\pm 0.39$ & $0.33\pm 0.00$ & $0.36\pm 0.03$ \\
%       Hodge & $0.25\pm 0.00$ & $0.26\pm 0.00$ & $0.12\pm 0.00$ & $0.12\pm 0.00$ & $0.87\pm0.00$ & $0.87\pm0.00$  & $0.48\pm0.01$ & $1.31\pm 0.69$ & $0.33\pm0.01$ & $0.37\pm 0.04$ \\ 
%       \bottomrule
%       \end{tabular}
%   }
% \end{table}

\begin{table}[t!]
  \centering
  \caption{Ocean current inference results.} \label{app-tab:ocean-flow-full-results}
  \resizebox{0.8\linewidth}{!}{
    \begin{tabular}{lllllll}
      \toprule
      \multirow{2}{*}{Method}  & \multicolumn{3}{c}{RMSE } & \multicolumn{3}{c}{NLPD}  \\ 
      \cmidrule(lr){2-4} \cmidrule(lr){5-7}  & Diffusion & \Matern~& Hodge Laplacian  &  Diffusion & \Matern~& Hodge Laplacian \\
      \midrule      
      Euclidean & $1.00\pm0.01$ & $1.00\pm0.00$ & --- & $1.42\pm0.01$ & $1.42\pm0.10$ & --- \\
      Line-Graph & $0.99\pm0.00$ & $0.99\pm0.00$ & --- & $1.41\pm0.00$ & $1.41\pm0.00$ & --- \\ 
      Non-HC & $0.35\pm0.00$ & $0.35\pm 0.00$ & $0.35\pm 0.00$ & $0.33\pm 0.00$ & $0.36\pm0.03$ & $0.33\pm0.01$\\
      HC & $0.34\pm 0.00$ & $0.35\pm 0.00$ & $0.35\pm 0.00$ & $0.33\pm0.01$ & $0.37\pm 0.04$ & $0.33\pm 0.01$\\ 
      \bottomrule
      \end{tabular}
  }
  %\vspace{-2mm}
\end{table}

\begin{figure*}[hpt!]
  %\vspace{0mm}
  \begin{subfigure}{0.5\linewidth}
    \includegraphics[width=1\linewidth]{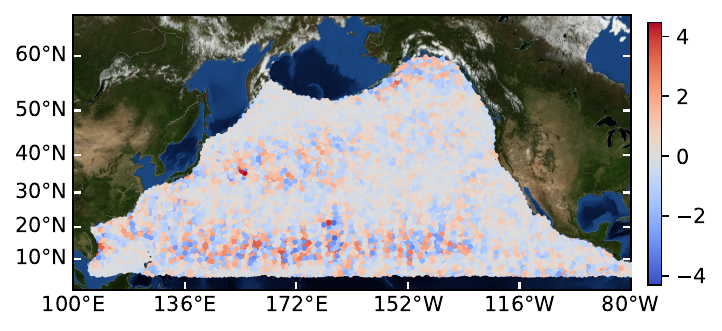}
    \caption{Original ocean current}
  \end{subfigure}
  \begin{subfigure}{0.5\linewidth}
    \includegraphics[width=1\linewidth]{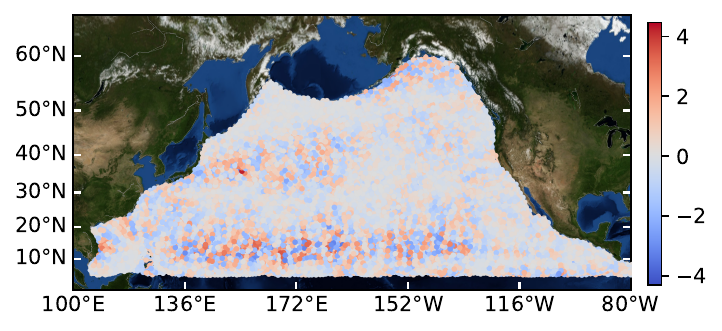}
    \caption{Predicted ocean current}
  \end{subfigure}
  \begin{subfigure}{0.5\linewidth}
    \includegraphics[width=1\linewidth]{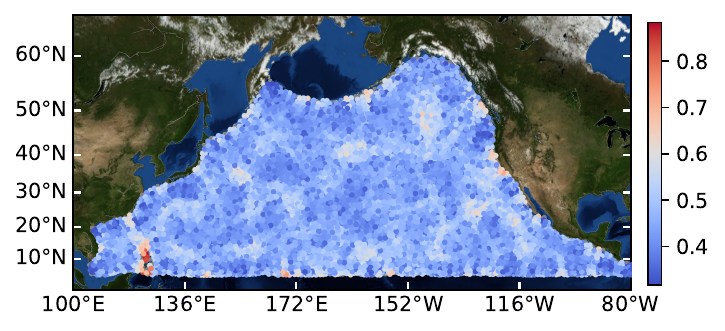}
    \caption{Standard deviation}
  \end{subfigure}
  \begin{subfigure}{0.5\linewidth}
    \includegraphics[width=1\linewidth]{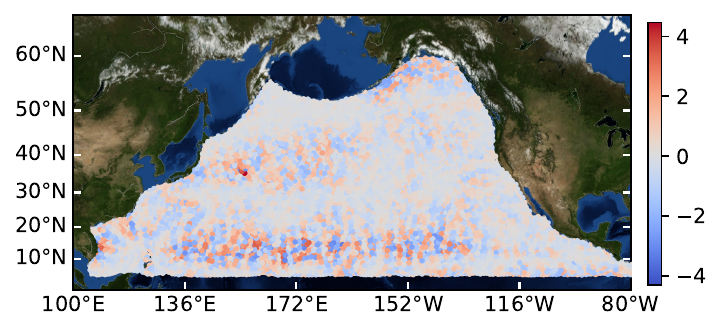}
    \caption{Predicted ocean current, non-HC diffusion}
    \label{app-fig:ocean-flow-sample-edge-flow}
  \end{subfigure}
  \begin{subfigure}{0.5\linewidth}
    \includegraphics[width=1\linewidth]{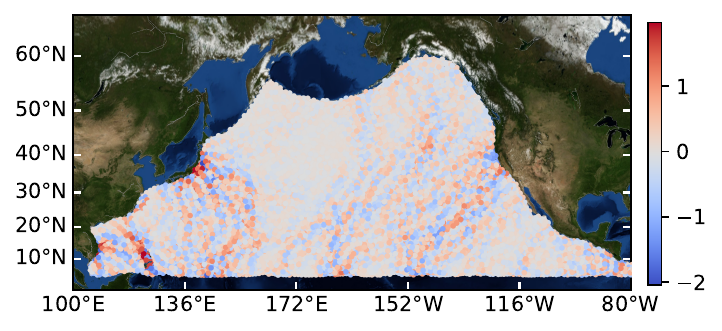}
    \caption{Original gradient ocean current}
  \end{subfigure}
  \begin{subfigure}{0.5\linewidth}
    \includegraphics[width=1\linewidth]{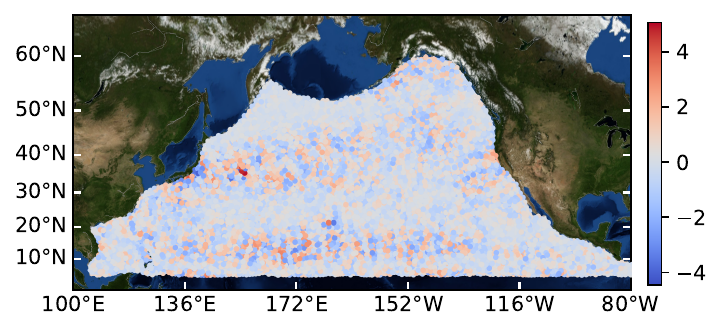}
    \caption{Original curl ocean current}
  \end{subfigure}
  \begin{subfigure}{0.5\linewidth}
    \includegraphics[width=1\linewidth]{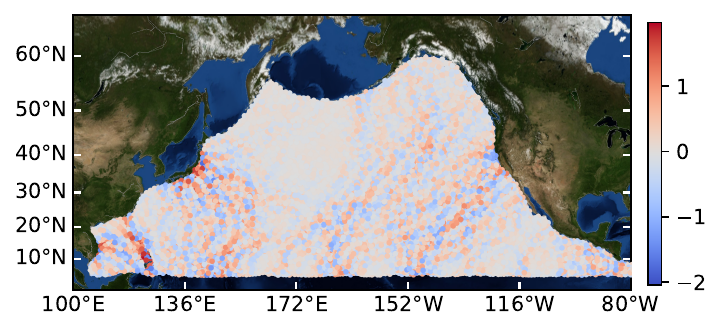}
    \caption{Predicted gradient flow}
  \end{subfigure}
  \begin{subfigure}{0.5\linewidth}
    \includegraphics[width=1\linewidth]{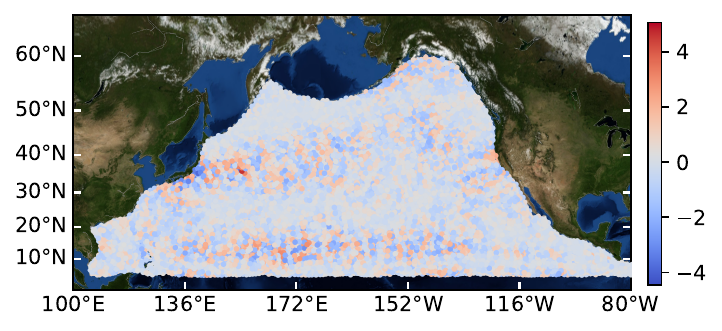}
    \caption{Predicted curl flow}
  \end{subfigure}
  \caption{(a-h) Results for ocean current prediction with 20\% training ratio in the edge flow domain. Note that we highlight the edge flow values on the middle points of the edges.}
  \label{app-fig:ocean-flow-full-figures-edge-flow}
\end{figure*}

\begin{figure*}[hpt!]
  %\vspace{0mm}
  \begin{subfigure}{0.5\linewidth}
    \includegraphics[width=1\linewidth]{figures/ocean_flow/noisy_vector_field.pdf}
    \caption{Original ocean current}
  \end{subfigure}
  % \begin{subfigure}{0.5\linewidth}
  %   \includegraphics[width=1\linewidth]{figures/ocean_flow/zero_padded_vector_field.pdf}
  %   \caption{Zero-padded prediction}
  % \end{subfigure}
  \begin{subfigure}{0.5\linewidth}
    \includegraphics[width=1\linewidth]{figures/ocean_flow/posterior_vector_field.pdf}
    \caption{Posterior mean}
  \end{subfigure}
  \begin{subfigure}{0.5\linewidth}
    \includegraphics[width=1\linewidth]{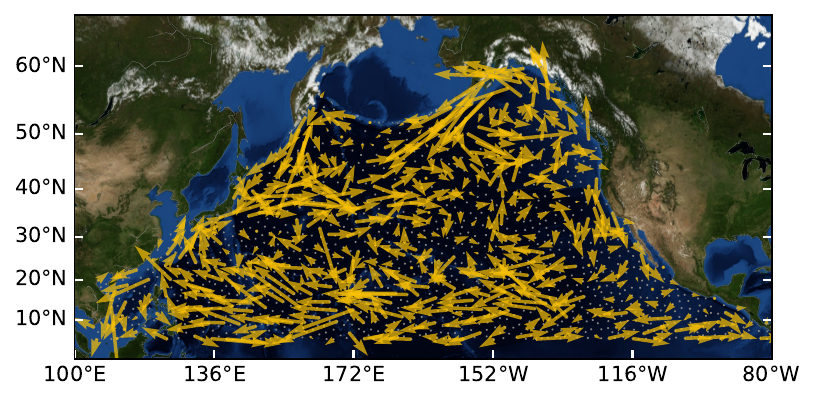}
    \caption{Sample from posterior distribution}
    \label{app-fig:ocean-flow-sample}
  \end{subfigure}
  \begin{subfigure}{0.52\linewidth}
    \includegraphics[width=1.02\linewidth]{figures/ocean_flow/std_pointwise_approx.pdf}
    \caption{Standard deviation approximated using 50 samples}
    \label{app-fig:std_pointwise_approx}
  \end{subfigure}
  \begin{subfigure}{0.5\linewidth}
    \includegraphics[width=1\linewidth]{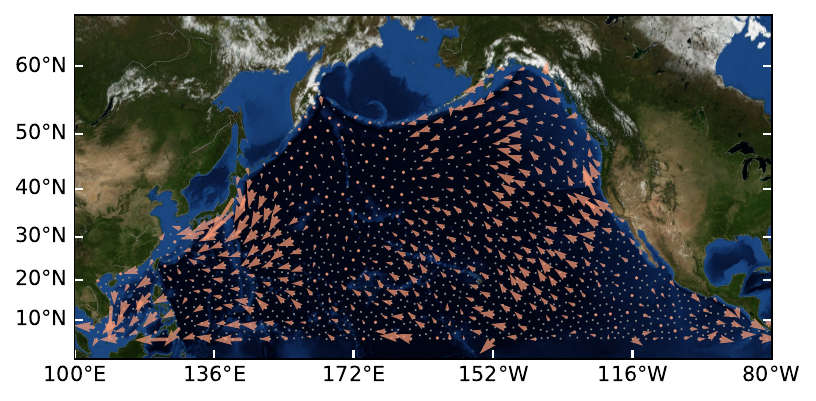}
    \caption{Original curl-free component}
  \end{subfigure}
  \begin{subfigure}{0.5\linewidth}
    \includegraphics[width=1\linewidth]{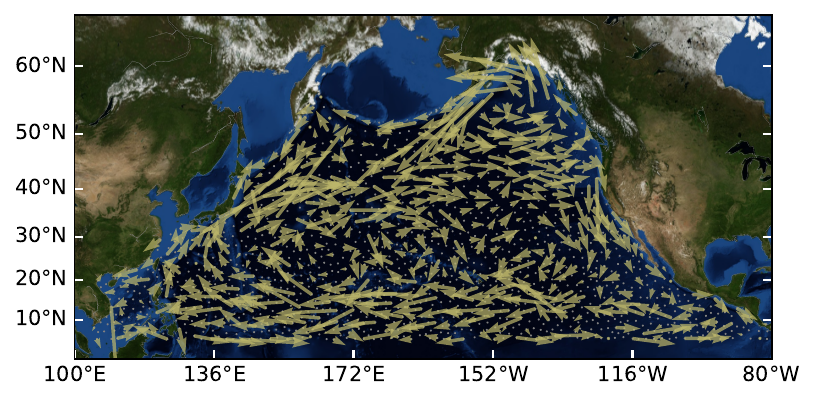}
    \caption{Original div-free component}
  \end{subfigure}
  \begin{subfigure}{0.5\linewidth}
    \includegraphics[width=1\linewidth]{figures/ocean_flow/grad_vector_field.pdf}
    \caption{Predicted curl-free component}
  \end{subfigure}
  \begin{subfigure}{0.5\linewidth}
    \includegraphics[width=1\linewidth]{figures/ocean_flow/curl_vector_field.pdf}
    \caption{Predicted div-free component}
  \end{subfigure}
  % \begin{subfigure}{0.47\linewidth}
  %   \includegraphics[width=0.95\linewidth]{figures/ocean_flow/learned_kernel_spectral.pdf}
  %   \caption{Learned HC diffusion kernel}
  % \end{subfigure}
  % \begin{subfigure}{0.5\linewidth}
  %   \includegraphics[width=1\linewidth]{figures/ocean_flow/std_vector_field.pdf}
  %   \caption{Standard deviation}
  % \end{subfigure}
  % \begin{subfigure}{0.49\linewidth}
  %   \includegraphics[width=0.95\linewidth]{figures/ocean_flow/learned_kernel_spectral.pdf}
  %   \caption{Learned HC diffusion kernel}
  % \end{subfigure}
  % \begin{subfigure}{0.33\linewidth}
  %   \includegraphics[width=1\linewidth]{figures/ocean_flow/smoothed_grad_vector_field.pdf}
  %   \caption{smoothed grad}
  % \end{subfigure}
  % \begin{subfigure}{0.33\linewidth}
  %   \includegraphics[width=1\linewidth]{figures/ocean_flow/smoothed_curl_vector_field.pdf}
  %   \caption{Smoothed curl}
  % \end{subfigure}
  \caption{(a-h) Results for ocean current prediction with 20\% training ratio in the vector field domain. Note that (a-b) and (g-h) are the same as in \cref{fig:ocean_flow_illustration}. We show them here for reader's convenience.}
  \label{app-fig:ocean-flow-full-figures}
\end{figure*}

% \newpage 
\subsection{Additional Details for Water Supply Networks}
% \begin{enumerate}
%   \item water equation explain 
%   \item network generation explain, as well as the data 
% \end{enumerate}
We obtain the Zhi Jiang WSN from \citet{dandy06} which contains 114 nodes (113 tanks and 1 source reservoir) and 164 edges (water pipes), no triangles considered. 
We build an unweighted graph based on the topology of this WSN. 
We model the hydraulic heads as functions on nodes $\vf_0$ and water flowrates as functions on edges $\vf_1$. 
A WSN is often governed by the following equations
\begin{equation}
  \text{mass conservation}: \mB_1\vf_1 = \vq, 
  \text{ and }
  \text{Hazen-Williams equation}: [\mB_1^\top\vf_0](e) = \bar{\vf}_1(e):=r_e f_1(e)^{1.852}
\end{equation}
for a pipe $e$, where $\vq\in\R^{N_0}$ is the demand on nodes, $r_e$ is the roughness of pipe $e$ \citep{dini2014new}.
% The roughness depends on the length, diameter and other constants of the pipe.
We then use the \texttt{WNTR} library \citep{klise2017water} to simulate a scenario generating the states of node heads and edge flowrates given the pipe roughnesses and the node demands.
The latter are sampled uniformly from $0$ to $10$ (unit $\rm{liter}/s$), modeling the read-world demand. 

We consider the joint state estimation of both heads (using node GPs) and the adjusted flowrates $\bar{\vf}_1$ (using edge GPs). 
Specifically, our GP models are 
\begin{equation}
  \begin{pmatrix}
    \vf_0 \\ 
    \bar{\vf}_1 
  \end{pmatrix}
  \sim \gG\gP
  \Bigg( 
  \begin{pmatrix}
    \vzero \\
    \vzero
  \end{pmatrix},
  \begin{pmatrix}
    \mK_0 & \\ 
    & \mK_1
  \end{pmatrix}
  \Bigg).
\end{equation}
We choose the \Matern~and diffusion node GPs \citep{borovitskiyMatErnGaussian2021}.
For HC edge GPs, we leverage the physical prior to model $\mK_1 = \mB_1^\top\mK_0\mB_1$ as discussed in \cref{cor:gradient-of-node-gp}, while for non-HC edge GPs, we choose them as in \cref{eq.simple_matern_diffusion}, of the same type as node GPs.
We randomly sample $50\%$ of the nodes and edges for training and use the rest for test. 
Note that the WSN has small edge connectivity. 
The randomness of the training set may disconnect the graph, which may deteriorate the performance, causing the large variance in the metrics.

\end{document}